\documentclass{article}

\usepackage{arxiv}

\usepackage{amssymb,amsmath,thmtools}
\usepackage{graphicx}
\usepackage{url}
\usepackage{setspace}
\usepackage{xcolor}
\usepackage{todonotes}
\usepackage{microtype}
\usepackage{graphicx}
\usepackage{subfigure}
\usepackage{booktabs} 
\usepackage{algorithm}
\usepackage[noend]{algorithmic}
\usepackage{multirow}
\usepackage{hyperref}
\hypersetup{hidelinks}

\usepackage{colortbl}
\usepackage{enumitem}
\definecolor{maroon}{cmyk}{0,0.87,0.68,0.32}

\usepackage{times}
\usepackage{mathtools, amsthm, amsfonts, thmtools}
\usepackage{cleveref}
\usepackage{graphicx}
\usepackage{natbib}

\usepackage{booktabs}
\usepackage{url}

\newcommand{\adv}{\mathrm{adv}}
\newcommand{\nat}{\mathrm{nat}}

\newtheorem{theorem}{Theorem}[section]
\newtheorem{lemma}[theorem]{Lemma}

\newtheorem{corollary}[theorem]{Corollary}
\newtheorem{remark}{Remark}

\newtheorem{definition}{Definition}

\renewcommand{\a}{\mathbf{a}}
\renewcommand{\b}{\mathbf{b}}
\renewcommand{\c}{\mathbf{c}}
\newcommand{\e}{\mathbf{e}}

\renewcommand{\u}{\mathbf{u}}
\renewcommand{\v}{\mathbf{v}}
\newcommand{\w}{\mathbf{w}}
\newcommand{\x}{\mathbf{x}}
\newcommand{\y}{\mathbf{y}}
\newcommand{\z}{\mathbf{z}}

\newcommand{\E}{\mathbf{E}}

\newcommand{\R}{\mathbb{R}}

\newcommand{\0}{\mathbf{0}}
\newcommand{\1}{\mathbf{1}}
\newcommand{\comment}[1]{}

\newcommand{\blue}[1]{{\color{black}#1}}
\newcommand{\proposed}[2]{{#1}}

\newcommand{\cA}{\mathcal{A}}
\newcommand{\cB}{\mathcal{B}}
\newcommand{\cC}{\mathcal{C}}
\newcommand{\cD}{\mathcal{D}}
\newcommand{\cE}{\mathcal{E}}

\newcommand{\cS}{\mathcal{S}}
\newcommand{\cT}{\mathcal{T}}

\newcommand{\cN}{\mathcal{N}}
\newcommand{\cO}{\mathcal{O}}

\newcommand{\cY}{\mathcal{Y}}
\newcommand{\cX}{\mathcal{X}}
\newcommand{\cZ}{\mathcal{Z}}
\newcommand{\bbB}{\mathbb{B}}
\newcommand{\bbE}{\mathbb{E}}

\newcommand{\dist}{\mathsf{dist}}
\newcommand{\spann}{\mathsf{span}}
\newcommand{\vol}{\mathsf{vol}}
\newcommand{\Null}{\mathsf{null}}
\newcommand{\Area}{\mathsf{Area}}
\newcommand{\ind}{\mathbb{I}}

\DeclareMathOperator*{\argmax}{argmax}
\DeclareMathOperator*{\argmin}{argmin}

\begin{document}

\title{An Analysis of Robustness of Non-Lipschitz Networks
}

\author{\textbf{Maria-Florina Balcan}\\
\texttt{ninamf@cs.cmu.edu} \\
        Carnegie Mellon University\\
        \and
		\textbf{Avrim Blum}\\
		\texttt{avrim@ttic.edu}\\
        Toyota Technological Institute at Chicago
        \AND
        \textbf{Dravyansh Sharma}\\
        \texttt{dravyans@cs.cmu.edu} \\
        Carnegie Mellon University
        \and
        \textbf{Hongyang Zhang}\\
        \texttt{hongyang.zhang@uwaterloo.ca}\\
        University of Waterloo}

\maketitle

\begin{abstract}%
Despite significant advances, deep networks remain highly susceptible to adversarial attack.  One fundamental challenge is that small input perturbations can often produce large movements in the network’s final-layer feature space.  In this paper, we define an attack model that abstracts this challenge, to help understand its intrinsic properties.  In our model, the adversary may move data an arbitrary distance in feature space but only in random low-dimensional subspaces.  We prove such adversaries can be quite powerful: defeating any algorithm that must classify any input it is given.  However, by allowing the algorithm to abstain on unusual inputs, we show such adversaries can be overcome when classes are reasonably well-separated in feature space. We further provide strong theoretical guarantees for setting algorithm parameters to optimize over accuracy-abstention trade-offs using data-driven methods. Our results provide new robustness guarantees for nearest-neighbor style algorithms, and also have application to contrastive learning, where we empirically demonstrate the ability of such algorithms to obtain high robust accuracy with low abstention rates.  Our model is also motivated by {\em strategic classification}, where entities being classified aim to manipulate their observable features to produce a preferred classification, and we provide new insights into that area as well.
\end{abstract}

\section{Introduction}
A substantial body of work has shown that deep networks can be highly susceptible to adversarial attacks, in which minor changes to the input lead to incorrect, even bizarre classifications~\citep{szegedy2014intriguing,moosavi2016deepfool,madry2018towards,su2019one,brendel2018adversarial}.  Much of this work has considered bounded $\ell_p$-norm attacks, though many other forms of attack are considered as well \citep{brown2018unrestricted,engstrom2017rotation,gilmer2018motivating,xiao2018spatially,alaifari2018adef}. What these results have in common is that changes that either are imperceptible or should be irrelevant to the classification task can lead to drastically different network behaviors.  

One key reason\footnote{\blue{Additional explanations include the presence of brittle features that are human incomprehensible \citep{ilyas2019adversarial}, and the location of the classification boundary relative to the submanifold of sampled data \citep{tanay2016boundary}.}} for this vulnerability to attacks is the non-Lipschitzness of typical neural networks: small but adversarial movements in the input space can produce large perturbations in the feature space \citep{yang2020closer,szegedy2014intriguing,goodfellow2014explaining}. This ability of an adversary to produce large movements in feature space appears to be at the heart of many of the successful attacks to date.  If we assume that non-Lipschitzness is important for good performance on natural data, then it is crucial to understand to what extent this property makes a network intrinsically susceptible to attacks.

\comment{
\begin{figure}[t]
\centering
\includegraphics[scale=0.97]{intro_figure.pdf}
\caption{Illustration of a non-Lipschitz feature mapping using a deep network. $A$: Clean image of CIFAR-10. $A'$: Corrupted image by 8-intensity-level adversarial perturbation. $B$: Clean image of another class. For a naturally trained deep network $F$ on CIFAR-10, we see that $\|F(A)-F(A')\|_2>\|F(A)-F(B)\|_2$, even though $A$ and $A'$ are human-indistinguishable.}
\label{figure: nl}
\end{figure}
}

In this work, we propose and analyze an abstract attack model designed to focus on this question of the intrinsic vulnerability of non-Lipschitz networks, and what might help to make such networks robust. In particular, suppose an adversary, by making an imperceptible change to an input $x$, can cause its representation $F(x)$ in feature space (the final layer of the network) to move by an arbitrary amount: will such an adversary always win?  Clearly if the adversary can modify $F(x)$ by an arbitrary amount {\em in an arbitrary direction}, then yes, because it can then move $F(x)$ into the classification region of any other class it wishes.  But what if the adversary can modify $F(x)$ by an arbitrary amount but only in a {\em random} direction or within a random low-dimensional subspace (which it cannot control)?   In this case, we show an interesting dichotomy: if the classifier must output a classification on any input it is given, then indeed the adversary will still win, no matter how well-separated the natural data points from different classes are in feature space and no matter what decision surface the classifier uses.  Specifically, for any data distribution and any decision surface, there must exist at least one class such that the adversary wins with significant probability on random examples of that class.  However, if we provide the classifier the ability to abstain, then we show it can defeat such an adversary (while maintaining a low abstention rate on natural data) with a nearest-neighbor style approach under fairly reasonable conditions on the distribution of natural data in feature space.  Moreover, we show these conditions can often be achieved using contrastive learning. Our results also hold for generalizations of these models, such as  directions that are not completely random. More broadly, our results provide a theoretical explanation for the importance of allowing abstaining, or selective classification, in the presence of adversarial attacks that exploit network non-Lipschitzness.  Our results also provide new understanding of the robustness of nearest-neighbor algorithms.

\proposed{A second motivation of our work comes from the area of {\em strategic classification}, where the concern is that entities being classified may try to manipulate their observable features to achieve a preferred outcome.  Consider, for example, a public rating system used to classify companies into those that are good to work for and those that are not.  Naturally, companies want to be viewed as being good to work for.  So, they may try to modify any easy-to-manipulate features used by the system in order to achieve a positive classification, even if this does not change their true status. For example, perhaps the system uses the ratio of managers to associates, which the company can manipulate arbitrarily (from 0 to infinity) by changing employee titles, without actually changing pay or responsibilities.  Suppose we assume agents (the companies) have a small number of parameters they can manipulate arbitrarily, and that there is an unknown linear function that maps changes in these parameters to movement in feature space. In this case, our results can provide some guidance.  Our negative results imply that for any non-abstaining classifier, there must be at least one class such that for most examples from that class, manipulation in a random direction has a significant chance of being successful; whereas our positive results imply that by using the ability to abstain, we can be secure against manipulation in most low-dimensional subspaces. Note that this is very similar to the model used by \cite{Kleinberg2018HowDC} (see also \citealt{Alon2020MultiagentEM,shavit2020causal}) who assume that the ``effort conversion matrix'' mapping changes in manipulable parameters to changes in observable features is {\em known} (or at least can be learned through experimentation, \citealp{shavit2020causal}); our results provide insight into what can be done when it is {\em unknown}, and the classifier must be fielded before any manipulations are observed.}{}

\comment{
A large amount of works in the adversarial robustness study small perturbation threat model.
One may argue that these small perturbations roughly correspond to ``human-imperceptible'' attacks, that is, the adversary is motivated to keep the attack bounded to avoid human detection. But a human-in-the-loop is a fairly limiting requirement for a scalable machine learning application, and partly defeats the purpose of automated robustness. Also one may consider applications of machine learning to domains like malware detection~\cite{chen2017adversarial}, where a human inspection is infeasible.
}

\comment{
In this work, we present a model where large perturbations are indeed possible to defend against. If we could verify some of our features at a cost, the adversary would be driven to perturb a small random subset of them. This motivates us to consider an adversarial model where the attacker can perturb within a random subspace of the feature space. We allow the adversary to perform arbitrarily large manipulations within this subspace. Intuitively if the features are well-clustered and separated in the embedding space, an adversary moving in a random direction from one cluster is unlikely to hit another cluster even with large movements.
}

In addition to providing a formal separation between algorithms that can abstain and those that cannot, our work also yields an interesting trade-off between robustness and accuracy~\citep{tsipras2019robustness,zhang2019theoretically,raghunathan2020understanding} for nearest-neighbor algorithms.  By controlling a distance threshold determining the rate at which the nearest-neighbor algorithm abstains, we are able to trade off (robust) precision against recall, and we provide results for how to provably optimize for such a trade-off using a data-driven approach. We also perform experimental evaluation in the context of contrastive learning \citep{he2020momentum, chen2020simple, khosla2020supervised}.

\blue{We acknowledge that our model is only an abstraction. Additionally, one can also consider relaxations of the Lipschitzness condition. We discuss some work along these lines in Section \ref{sec:related}.}

\comment{
We present a nearest-neighbor based algorithm which obtains high robustness under our threat model. We further establish that under practical conditions satisfied by the feature embedding, our algorithm does not say ``don't know'' too often. We demonstrate the effectiveness of data-driven algorithm design in navigating the trade-off between robust accuracy and abstention. Our theoretical results provide useful practical guidelines for determining and constructing ``good'' feature spaces that enable adversarially robust classification. We supplement our theory with empirical applications to contrastive learning.
}

\subsection{Our Contributions}

Our main contributions are the following.
Conceptually, we introduce a new \emph{random feature subspace} threat model to abstract the effect of non-Lipschitzness in deep networks. Technically, we show the power of abstention and data-driven algorithm design in this setting, proving that classifiers with the ability to abstain are provably more powerful than those that cannot in this model, and giving provable guarantees for nearest-neighbor style algorithms and data-driven hyperparameter learning.  Experimentally, we show that our algorithms perform well in this model on representations learned by supervised and self-supervised contrastive learning.  More specifically,

\begin{itemize}[leftmargin=*,topsep=0pt,partopsep=1ex,parsep=1ex]\itemsep=-4pt

\item We introduce the \emph{random feature subspace} threat model, an abstraction designed to  focus on the impact of non-Lipschitzness  on vulnerability to adversaries.

\item
We show for this threat model that \emph{all} classifiers that partition the feature space into two or more classes---without an ability to abstain---are provably vulnerable to adversarial attacks.  In particular, no matter how nice the data distribution is in feature space, for at least one class the adversary succeeds with significant probability.

\item
We show that in contrast, a classifier with the ability to abstain can overcome this vulnerability. We present a thresholded nearest-neighbor algorithm that is provably robust in this model when classes are sufficiently well separated, and characterize the conditions under which the algorithm does not abstain too often. This result can be viewed as providing new robustness guarantees for nearest-neighbor style algorithms as well as for \emph{proof-carrying predictions}, where predictions are accompanied by certificates of confidence.

\item

We leverage and extend dispersion techniques from data-driven algorithm design, and present a novel data-driven method for learning data-specific hyperparameters in our defense algorithms to simultaneously obtain high robust accuracy and low abstention rates. Unlike typical hyperparameter tuning, our approach provably converges to a global optimum.

\item
Experimentally, we show that our proposed algorithm achieves \emph{certified} adversarial robustness on representations learned by supervised and self-supervised contrastive learning. Our method significantly outperforms algorithms without the ability to abstain.

\end{itemize}
Our framework can be thought of as a kind of smoothed analysis  \citep{spielman2004smoothed} in its combination of random and adversarial components. This is especially so for Section \ref{sec:bounded-density} where we broaden our guarantees to apply to arbitrary $\kappa$-bounded distributions. However, a key distinction is that in smoothed analysis, the adversary moves first, and randomness is added to its decision afterwards.  In our model, in contrast, first a random restriction is applied to the space of perturbations the adversary may choose from, and then the adversary may move arbitrarily in that random subspace. Thus, the adversary in our setting has more power, because it can make its decision after the randomness has been applied.

\subsection{Related Work}
\label{sec:related}
\noindent{\textit{Large-magnitude adversarial perturbations.}} 
While most work on adversarial robustness  considers small perturbations (for example, \citealt{szegedy2014intriguing,madry2018towards,zhang2019theoretically}), there has also been significant work on other kinds of attacks such as adversarial rotations, translations, and deformations~\citep{brown2018unrestricted,engstrom2017rotation,gilmer2018motivating,xiao2018spatially,alaifari2018adef}.  Perhaps most closely related to our negative results in Section \ref{sec:negative} is work of \citet{shamir2019simple}.  \citet{shamir2019simple} consider an adversary that can make small $\ell_0$ perturbations in the input space: that is, perturb a small number of input coordinates, but change them by an arbitrary amount. They present algorithms {\em for the adversary} giving targeted attacks against any learner that partitions space with a hyperplane partition using a limited number of hyperplanes in general position.  Our negative results for non-abstaining classifiers are inspired by their work, though they are formally incomparable (our results are stronger in that they hold even if an adversary can just change one random linear combination and for an {\em arbitrary} partition of space, but weaker in that we consider an untargeted adversary, and different in that we assume randomness in the direction of adversarial power rather than in the space partition).  \citet{shamir2019simple} do not consider the use of abstention to combat this adversarial power. We discuss further connections to coordinate-wise perturbations in Section \ref{sec:comparison-coordinatewise}. \blue{\cite{shafahi2019are} look at the effect of dimensionality on robustness limits for $\ell_p$-norm bounded attacks, but their negative results do not hold for abstentive classifiers.}

\medskip
\blue{\noindent{\textit{Network Lipschitzness and relaxed notions.}} We model non-Lipschitzness of the network mapping in the context of robustness via large adversarial feature space movements corresponding to small input space perturbations. Several relaxations of the Lipschitz condition have been studied in the literature including H\"{o}lder smoothness \citep{an2021generalization}, local Lipschitzness \citep{hein2017formal,yang2020closer} and probabilistic Lipschitzness  \citep{urner2013probabilistic}. Typically satisfying these relaxed conditions leads to better performance than bounding the global Lipschitzness of the networks \citep{cisse2017parseval}.}

\medskip
\noindent{\textit{Adversarial robustness with abstention options.}}
Classification with abstention options (a.k.a. selective classification~\citep{geifman2017selective}) is a relatively less explored direction in the adversarial machine learning literature. \cite{hosseini2017blocking} augmented the output class set with a NULL label and trained the classifier to reject the adversarial examples by classifying them as NULL; \cite{stutz2020confidence} and \cite{laidlaw2019playing} obtained robustness by rejecting low-confidence adversarial examples according to confidence thresholding or predictions on the perturbations of adversarial examples. Another related line of research to our method is the detection of adversarial examples~\citep{grosse2017statistical,li2017adversarial,carlini2017adversarial,meng2017magnet,metzen2017detecting,bhagoji2018enhancing,xu2017feature,hu2019new,liu2018adversarial,Deng_2021_CVPR}. \blue{This direction also often involves thresholding a heuristic confidence score. For example, \cite{ma2018characterizing} use a confidence metric based on $k$-nearest neighbors in the training sample, and \cite{lee2018simple} fit class-wise Gaussian distributions and flag test points away from all distributions. These approaches have been studied empirically  but typically lack formal guarantees.}
\cite{goldwasser2020beyond}, on the other hand, gave provable guarantees for selective classification in a transductive setting in which performance was measured according to an adversarial test distribution from which unlabeled examples are provided to the learning algorithm in advance.

\medskip
\noindent{\textit{Data-driven algorithm design.}} Data-driven algorithm design refers to using machine learning for algorithm design, including choosing a good algorithm from a parameterized family of algorithms for given data. It is known as ``hyperparameter tuning'' to machine learning practitioners and typically involves a ``grid search'', ``random search'' \citep{bergstra2012random} or gradient-based search, with no guarantees of convergence to a global optimum.

Data-driven algorithm design was formally introduced to the theory of computing community by \cite{gupta2017pac} as a learning paradigm, and was further extended in \cite{balcan2017learning}. The key idea is to model the problem of identifying a good algorithm from data as a statistical learning problem. The technique has found useful application in providing provably better algorithms for several problems of fundamental significance in machine learning including clustering \citep{balcan2020learning,balcan2018data,balcan2021learning}, semi-supervised learning \citep{balcan2021data}, simulated annealing \citep{blum2021learning}, regularized regression \citep{balcan2022provably}, mixed integer programming \citep{balcan2018learning,balcan2022structural}, low rank approximation \citep{bartlett2022generalization} and even beyond, providing guarantees like differential privacy and adaptive online learning \citep{balcan2018dispersion,sharma2020learning}.
See \cite{balcan2020data} for further discussion on this rapidly growing body of research. For learning in an adversarial setting, we provide the first demonstration of the effectiveness of data-driven algorithm design in a defense method 
to optimize over the accuracy-abstention trade-off with strong theoretical guarantees.

\proposed{
\medskip
\noindent{\textit{Strategic classification.}} Strategic classification considers the case that entities being classified have a stake in the outcome, and will aim to manipulate their observable features to receive the classification they desire.  Typically it is assumed these entities have some limited power to manipulate, and that this power is known to the classifier. \cite{chen2020learning,ahmadi2021strategic} consider entities that can manipulate inside a ball of some limited radius, whereas \citet{Kleinberg2018HowDC,Alon2020MultiagentEM,shavit2020causal} consider agents that have ``activities'' they can engage in at some cost, which get converted into movement in feature space via an ``effort conversion matrix''. This latter work assumes the effort conversion matrices are known, or at least can be learned from experimentation.  In contrast, our setting can be viewed as a case where the matrices are unknown and the classifier must be fielded before any manipulations are observed (and agents have an unlimited activity budget).  Note, the work of \citet{Kleinberg2018HowDC,Alon2020MultiagentEM,shavit2020causal} also considers the case that only certain activities correspond to ``gaming'' and others correspond to true self-improvement; we do not consider the self-improvement aspect here.}{}

\blue{
\medskip
\noindent{\textit{Adversarial defenses by non-parametric methods.}}
Adversarial defenses by $k$-nearest neighbor classifier have received significant attention in recent years. In the setting of norm-bounded threat model without the ability to abstain, \cite{wang2018analyzing} showed that the robustness properties of $k$-nearest neighbors depend critically on the value of $k$---the classifier may be inherently non-robust for small $k$, but its robustness approaches that of the Bayes Optimal classifier for fast-growing $k$. \cite{yang2020robustness,bhattacharjee2020non} provided and analyzed a general defense method, adversarial pruning, that works by preprocessing the data set to become well-separated and then running $k$-nearest neighbors. Theoretically, they derived an optimally robust classifier, which is analogous to the Bayes Optimal, and showed that adversarial pruning can be viewed as a finite sample approximation to this optimal classifier. In this work, we study the power of 1-nearest neighbors for adversarial robustness with the ability to abstain, under a random-subspace adversarial threat model.

\medskip
\noindent{\textit{Feature-space attacks.}}
Different from most existing attacks that directly perturb input pixels, there are a few prior works that focus on perturbing abstract features as ours. More specifically, the subspaces of features typically characterize styles, which include interpretable styles such as vivid colors and sharp outlines, and uninterpretable ones~\citep{xu2020towards}. \cite{ganeshan2019fda} proposed a \emph{feature disruptive attack} that generates an image perturbation that disrupts features at each layer of the network and causes deep-features to be highly corrupt. They showed that the attacks generate strong adversaries for image classification, even in the presence of various defense measures. Despite a large amount of empirical works on adversarial feature-space attack, many fundamental questions remain open, such as developing a \emph{provable} defense against feature-space attacks.

\medskip
\noindent{\textit{Learning with noise.}}
Classic work on learning with noise is a related line of work with theoretical guarantees \citep{kearns1988learning,bshouty2002pac,awasthi2014power}. These models typically  involve perturbations of  input-space features of training points. Our nearest-neighbor based techniques for test-time feature-space attacks are different from the localization and disagreement-based approaches that are known to work for poisoning attacks \citep{awasthi2014power,awasthi2016learning,balcan2022robustly}. An interesting direction for future work is to determine how to adapt our techniques to handle noise in data.
}

\section{Preliminaries and Threat Model}

\noindent\textit{Notation.} We will use \emph{bold lower-case} letters such as $\x$ and $\y$ to represent vectors, \emph{lower-case} letters such as $x$ and $y$ to represent scalars, and \emph{calligraphic capital} letters such as $\cX$, $\cY$ and $\cD$ to represent distributions. Specifically, we denote by $\x\in\cX$ a sample instance, and by $y\in\cY$ a label, where $\cX\subseteq \R^{n_1}$ and $\cY$ indicate the image and label spaces, respectively. Let $F:\cX\rightarrow \R^{n_2}$ be our given \emph{feature embedding} (which we assume has already been learned) that maps an instance to a high-dimensional vector in the latent space $F(\cX)$. It can be parameterized, for example, by deep neural networks. We will frequently use $\v\in\R^{n_2}$ to represent an adversarial perturbation in the feature space. Denote by $\dist(\z,\z')$ the Euclidean distance between any two vectors $\z,\z'$ in the image or feature space,
and let $\bbB(\z,\tau) = \{\z':\dist(\z,\z')\le\tau\}$ be the ball of radius $\tau$ about $\z$. We will use $\cD_{\cX}$ to denote the distribution of instances in the input space, $\cD_{\cX|y}$ the distribution of instances in the input space conditioned on the class $y$, $\cD_{F(\cX)}$ the distribution of instances in feature space, and $\cD_{F(\cX)|y}$ the distribution of instances in feature space conditioned on the class $y$.
Finally, we will typically use $(\x_1,y_1),...,(\x_m,y_m)$ to denote a given set of $m$ labeled training examples. 

\subsection{The Random Feature Subspace Threat Model}
\label{sec: threat model}

We now formally present the {\em random feature subspace} threat model, in which the adversary, by making small changes in the input space, is assumed to be able to create arbitrarily large movements in feature space, though only in random low-dimensional subspaces.  Note that because this large modification in feature space is assumed to come from a small perturbation in input space, we always assume that the {\em true correct label $y$ is the same for the modified point and the original point}.  

Specifically, let $\x$ be an $n_1$-dimensional test input for classification. The input is embedded into an $n_2$-dimensional feature space using an abstract mapping $F$. Our threat model is that the adversary may corrupt $F(\x)$ such that the modified feature vector is any point in a random $n_3$-dimensional affine subspace denoted by $\cS+\{F(\x)\}$. For example, if $n_3=1$ then $\cS+\{F(\x)\}$ is a random line through $F(\x)$, and the adversary may select an arbitrary point on that line; if $n_3=2$ then $\cS+\{F(\x)\}$ is a random 2-dimensional plane through $F(\x)$, and the adversary may select an arbitrary point in that plane. Conceptually, we are viewing $F$ as ``squashing'' the adversarial ball about $\x$ in input space into a random infinitely thin and infinitely wide $n_3$-dimensional pancake in feature space.  The adversary is given access to everything including the algorithm's classification function, $F$, $\x$, $\cS$ and the true label of $\x$. Throughout the paper, we will use \emph{adversary} and \emph{adversarial example} to refer to this threat model.

\subsection{Discussion and Examples}
As noted above, we are viewing the network as conceptually squashing the ball about $\x$ in input space into a random infinitely wide $n_3$-dimensional pancake in feature space. Of course, in a real network there would be some limit on the magnitude of a perturbation in feature space, and the available directions wouldn't exactly form a subspace. However, we believe this is a clean theoretical model worthy of understanding for insight.  Also, it is interesting to note that our negative results for non-abstaining classifiers, such as Theorem \ref{theorem: negative results}, apply even if $n_3=1$, whereas our positive results for classifiers that can abstain, such as Theorem \ref{theorem: positive results}, apply even if $n_3\geq n_1$ so long as $n_3 \ll n_2$.  

\smallskip

\noindent\textit{An example of a non-Lipschitz mapping.}  While our threat model is intended to be an abstraction, here is an example of a concrete non-Lipschitz mapping captured by our model.  Let us say the support of the natural data distribution $\cD$ only includes points with integer coefficients (data is in $\mathbb{R}^d$, but all natural points have integer coordinates). Assume the adversary can move points in the input space within an $\ell_2$ ball of radius 1/4. Now, for a point $x$, let us define $\text{frac}(x)$ to be its fractional part and $\text{int}(x)$ to be its integral part.  So if $d=2$ and $x = (1.1, 2.3)$ then frac$(x) = (0.1, 0.3)$ and $\text{int}(x)=(1,2)$.  If $x = (0.9, 0.8)$ then frac$(x) = (-0.1, -0.2)$ and int$(x)=(1,1)$. Now, let us say the network maps a point $x$ to $F(x)=x + \langle w,\text{frac}(x)\rangle w$, where $w$ is a large random vector (chosen independently at random for each lattice point int$(x)$). Then, all natural data will stay where they are (this is the identity mapping on natural data), but points in the adversarial ball can move very far in the direction of their $w$. So, if the true decision boundary is, say $x_1 \geq 1/2$, the adversary will not change the true label of any data point but (in the limit as $|w| \rightarrow\infty$) will be able to defeat any non-abstaining classifier in the feature space by Theorem \ref{theorem: negative results}.

\smallskip

\noindent\textit{Additional remarks.}
We wish to be clear that our intent is not to create a threat model against which one would design a new network architecture or training procedure.  Instead, we are thinking of a network that has already been trained (say, using adversarial training or any of the other available methods that try to improve robustness). But, the designers are finding that the adversarial loss is unacceptably high, because for many test points, the adversary can still move those points a large distance in feature space and cross over their decision boundary (even if the natural data of different classes are well-separated in feature space).  Our framework is aimed to consider this setting, and our results provide a practical suggestion: modify the final level to allow it to abstain if a test point is ``too different'' from the training data.  The justification is that if the adversary can move large distances but not in every possible direction (if it can do that, then no defense will work) and indeed only do so in random lower-dimensional subspaces, then we can provide theoretical guarantees for this approach. Moreover, our lower bounds show that abstention is {\em necessary} no matter how nicely distributed the data may be.  In fact, it is necessary even if the adversary can move points arbitrarily large distances in feature space even in just a single random direction.

\section{Negative Results without an Ability to Abstain}
\label{sec:negative}

We now present a hardness result showing that no matter how nicely data is distributed in feature space (for example, even if the network perfectly clusters data by label in feature space), any classifier that is not allowed to abstain will fail against our threat model even for an adversary that can perturb points in a single random direction ($n_3=1$).

\begin{theorem}
\label{theorem: negative results}
For any classifier that partitions $\R^{n_2}$ into two or more classes, any data distribution $\cD$, any $\delta > 0$ and any feature embedding $F$, there must exist at least one class $y^*$, such that for at least a $1-\delta$ probability mass of examples $\x$ from class $y^*$ (that is, $\x$ is drawn from $\cD_{\cX|y^*}$), for a random unit-length vector $\v$, with probability at least $1/2 - \delta$ for some $\Delta_0>0$, $F(\x) + \Delta_0 \v$ is not labeled $y^*$ by the classifier. In other words, there must be at least one class $y^*$ such that for at least $1-\delta$ probability mass of points $\x$ of class $y^*$, the adversary wins with probability at least $1/2-\delta$.
\end{theorem}

\begin{proof}
~Define $r_\delta$ to be a radius such that in the feature space, for every class $y$, at least a $1-\delta$ probability mass of examples $F(\x)$ of class $y$ lie within distance $r_\delta$ of the origin.  Define $R$ such that for a ball of radius $R$, if we move the ball by a distance $r_\delta$, at least a $1-\delta$ fraction of the volume of the new ball is inside the intersection with the old ball. 
Now, let $\cB$ be the ball of radius $R$ centered at the origin in feature space.  Let $\vol(\cB)$ denote the volume of $\cB$ and let $\vol_y(\cB)$ denote the volume of the subset of $\cB$ that is assigned label $y$ by the classifier.  Let $y^*$ be any label such that $\vol_{y^*}(\cB) / \vol(\cB) \leq 1/2$. 
Now by the definition of $y^*$, a point $\z$ picked uniformly at random from $\cB$ has probability at least $1/2$ of being classified differently from $y^*$.  This implies that, by the definition of $R$, if $F(\x)$ is within distance $r_\delta$ of the origin, then a point $\z_x$ that is picked uniformly at random in the ball $\cB_x$ of radius $R$ centered at $F(\x)$ has probability at least $1/2 - \delta$ of being classified differently from $y^*$.  This immediately implies that if we choose a random unit-length vector $\v$, then with probability at least $1/2 - \delta$, there exists $\Delta_0>0$ such that $F(\x) + \Delta_0 \v$ is classified differently from $y^*$, since we can think of choosing $\v$ by first sampling $\z_x$ from $\cB_x$ and then defining $\v = (\z_x - F(\x))/\|\z_x - F(\x)\|_2$.  Moreover, since the classifier has no abstention region, being classified differently from $y^*$ implies a win by the adversary. So, the theorem follows from the fact that, by the definition of $r_\delta$, at least $1-\delta$ probability mass of examples $F(\x)$ from class $y^*$ are within distance $r_\delta$ of the origin in feature space.
\end{proof}

\noindent We remark that our lower bound applies to any classifier and exploits the fact that a classifier without abstention must label the entire feature space. For a simple linear decision boundary, a perturbation in any direction (except parallel to the boundary) can cross the boundary with an appropriate magnitude (Figure \ref{figure: lb}, mid). The left and right figures show that if we try to `bend' the decision boundary to `protect' one of the classes, the other class is still vulnerable. Our argument formalizes and generalizes this intuition, and shows that there must be at least one vulnerable class irrespective of how you may try to shape the class boundaries, where the adversary succeeds in a large fraction of directions.
\begin{figure}[t]
\centering

\includegraphics[scale=0.65]{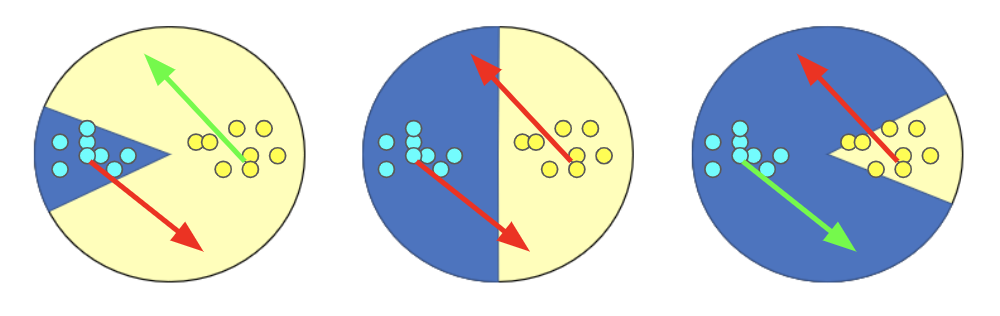}
\caption{A simple example to illustrate Theorem \ref{theorem: negative results}. Bending the decision boundary to avoid adversarial examples for one class makes it harder to defend the other class.\label{figure: lb}}

\end{figure}

\subsection{Comparison to Coordinate-wise Perturbations}\label{sec:comparison-coordinatewise}
It is interesting to compare our model to one in which the adversary can make only {\em coordinate-wise} perturbations in the feature space.  An adversary that can only make coordinate-wise changes would, in contrast, {\em not} be powerful enough to defeat any non-abstaining classifier.  For example, consider data in $\R^3$ where all the positive examples are at location $(1,1,1)$ and all the negative examples are at location $(-1,-1,-1)$ in feature space.  Then so long as the classifier partitions the space such that the lines $(1,1,\cdot)$, $(1, \cdot, 1)$, and $(\cdot, 1, 1)$ are positive and $(-1,-1,\cdot)$, $(-1,\cdot,-1)$, and $(\cdot, -1,-1)$ are negative, the adversary will not be able to defeat it with a single coordinate-wise change.  (We need to use $\R^3$ here rather than $\R^2$, because in $\R^2$ these lines would cross and so the classifier would not be well-defined). In contrast, by Theorem \ref{theorem: negative results}, an adversary that can perturb in a uniformly-random direction {\em will} defeat any non-abstaining classifier.

\section{Positive Results with an Ability to Abstain}\label{sec: positive results}

Theorem \ref{theorem: negative results} gives a hardness result for robust classification without abstention. In this section, we give positive results for a nearest-neighbor style classifier that has the power to abstain.

Given a test instance $\x\sim\cD_\cX$,
recall that the adversary is allowed to corrupt $F(\x)$ with an arbitrarily large perturbation in a uniformly-distributed subspace $S$ of dimension $n_3$. Consider the prediction rule that classifies the unseen example $F(\x)\in\R^{n_2}$ with the class of its nearest training example provided that the distance between them is at most $\tau$; otherwise the algorithm outputs ``don't know'' (see Algorithm \ref{algorithm: robust separated inference-time classifier} when $\sigma=0$). 
The threshold parameter $\tau$ trades off robustness against abstention rate; when $\tau\rightarrow\infty$, our algorithm is equivalent to the nearest-neighbor algorithm. Note that Algorithm \ref{algorithm: robust separated inference-time classifier} also contains a parameter $\sigma$ to remove training points that are too close to other training points of a different class---we will consider non-zero values of this parameter later. 

\begin{algorithm}[ht]
\caption{$\textsc{RobustClassifier}(\tau, \sigma)$}
\label{algorithm: robust separated inference-time classifier}
\begin{algorithmic}[1]
\STATE {\bfseries Input:} A test example $F(\x)$ (potentially an adversarial example), a set of training examples $F(\x_i)$ and their labels $y_i$, $i\in[m]$, a threshold parameter $\tau$, a separation parameter $\sigma$.
\STATE{\bfseries Preprocessing:} Delete training examples $F(\x_i)$ if $\min_{j\in[m],y_i\ne y_j}\dist(F(\x_i),F(\x_j))<\sigma$.
\STATE {\bfseries Output:} A predicted label of $F(\x)$, or ``don't know''.
\IF{$\min_{i\in[m]} \dist(F(\x),F(\x_i))<\tau$}
\STATE{\textbf{return} $y_{\argmin_{i\in[m]} \dist(F(\x),F(\x_i))}$;}
\ELSE
\STATE{\textbf{return} ``don't know''.}
\ENDIF
\end{algorithmic}
\end{algorithm}

\noindent Denote by $\cE_\adv^{\x}(f):=\mathbb{E}_{S\sim\cS}\1\{\exists \e\in S+F(\x) \text{ s.t. } f(\e)\ne\y \text{ and $f(\e)$ does not abstain}\}$ the robust error of a given classifier $f$ for classifying instance $\x$.
Our analysis leads to the following positive results on this algorithm. This theorem states that so long as the threshold $\tau$ is sufficiently small compared to the distance $r$ between the test point $\x$ and the nearest training point $\x_i$ of a different class (see Figure \ref{figure: meta-algorithm}), and the dimension $n_2$ of the ambient feature space is sufficiently large compared to the dimension $n_3$ of the adversarial subspace $S$, the algorithm will have low robust error on $\x$.

\begin{figure}[ht]
\center
\includegraphics[scale=0.65]{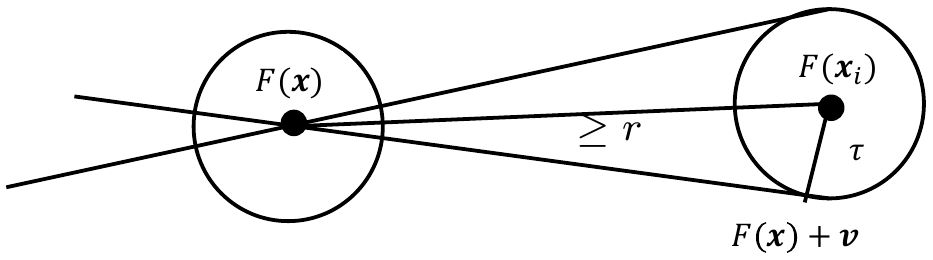}
\caption{Adversarial misclassification for nearest-neighbor predictor. Here $F(\mathbf{x})$ is the test point and $F(\mathbf{x}_i)$ is a training point from a different class. For $n_3=1$, the adversary succeeds for the directions inside the depicted cone around $F(\mathbf{x}_i)$.}
\label{figure: meta-algorithm}
\end{figure}

\begin{theorem}
\label{theorem: positive results}
Let $\x\sim\cD_\cX$ be a test instance, $m$ be the number of training examples and $r$ be the shortest distance between $F(\x)$ and $F(\x_i)$ where $\x_i$ is a training point from a different class. Suppose  $\tau=o\left(r\sqrt{1-\frac{n_3}{n_2}}\right)$. The robust error of Algorithm \ref{algorithm: robust separated inference-time classifier}, $\cE_\adv^{\x}(\textsc{RobustClassifier}(\tau,0))$, is at most $$m\Bigg(\frac{c\tau}{r\sqrt{1-\frac{n_3}{n_2}}}\Bigg)^{n_2-n_3}+mc_0^{n_2-n_3},$$ where $c>0$ and $0<c_0<1$ are absolute constants. For the case $n_3=1$, the robust error is at most $$m\left(\frac{\tau}{r}\right)^{n_2-1}.$$
\end{theorem}

\begin{proof}
We begin our analysis with the case of $n_3=1$. Suppose we have a training example $\x'$ of another class, and suppose $F(\x)$ and $F(\x')$ are at distance $D$ in the feature space. That is, $\dist(F(\x),F(\x'))=D$. Because $\tau=o\left(D\right)$, the probability that the adversary can move $F(\x)$ to within distance $\tau$ of $F(\x')$ is at most the ratio of the surface area of a sphere of radius $\tau$ to the surface area of a sphere of radius $D$, which is at most $$\left(\frac{\tau}{D}\right)^{n_2-1}\le \left(\frac{\tau}{r}\right)^{n_2-1}$$ if the feature space is $n_2$-dimensional.  See Figure \ref{figure: meta-algorithm}.

The analysis for the case of general values of $n_3$ follows from a peeling argument.  For this, we need the following Random Projection Theorem  \citep{dasgupta2003elementary,vempala2005random}. 
\begin{lemma}[Random Projection Theorem]
\label{lemma: Random Projection Theorem}
Let $\z$ be a fixed unit length vector in $d$-dimensional space and $\widetilde \z$ be the projection of $\z$ onto a random $k$-dimensional subspace. For $0<\delta<1$,
\begin{equation*}
\Pr\left[\left|\|\widetilde \z\|_2^2-\frac{k}{d}\right|\ge \delta \frac{k}{d}\right]\le e^{-\frac{k(\delta^2\blue{-\delta^3})}{4}}.
\end{equation*}
\end{lemma}
Without loss of generality, we assume $F(\x)=\0$ in $\R^{n_2}$. Next, note that the random subspace in which the adversary vector is restricted to lie can be constructed by the following sampling scheme: we first sample a vector $\v_1$ uniformly at random from a unit sphere in the ambient space $\R^{n_2}$ centered at $\0$; fixing $\v_1$, we then sample a vector $\v_2$ uniformly at random from a unit sphere in the nullspace of $\spann\{\v_1\}$; we repeat this procedure $n_3$ times and let $\spann\{\v_1,\v_2,...,\v_{n_3}\}$ be the desired subspace. Note that the sampling scheme satisfies the random adversary model. For the fixed nullspace $\Null(\spann\{\v_1,...,\v_i\})$ of dimension $n_2-i$, according to the analysis of the case $n_3=1$, if we condition on the distance $D_i$ between $F(\x)$ and $F(\x')$ when they are projected to $\Null(\spann\{\v_1,...,\v_i\})$, the probability over the draw of $\v_{i+1}$ of failure with respect to $\x'$ is at most $(\cO(\tau/D_i))^{n_2-i-1}$. We also note that $\Null(\spann\{\v_1,...,\v_i\})$ is a random subspace of dimension $n_2-i$. Thus by Lemma \ref{lemma: Random Projection Theorem} (with constant $\delta$), we have $D_i\ge Cr \sqrt{\frac{n_2-i}{n_2}}$ with probability at least $1-e^{-c'(n_2-i)}$, where $C,c'>0$ are absolute constants.
Therefore, by the union bound over the choice of $n_3$ nullspaces and the failure probability of the event $D_i\ge Cr\sqrt{\frac{n_2-i}{n_2}}$, the failure probability of the algorithm over $x'$ is at most
\begin{equation*}
\sum_{i=1}^{n_3}e^{-c'(n_2-i)}+\sum_{i=1}^{n_3}\left(\cO\left(\frac{\tau}{Cr\sqrt{\frac{n_2-i}{n_2}}}\right)\right)^{n_2-i}\le c_0^{n_2-n_3}+\left(\frac{c\tau}{r\sqrt{\frac{n_2-n_3}{n_2}}}\right)^{n_2-n_3}.
\end{equation*}
By the union bound over all $m$ training data points $\x'$ completes the proof.
\end{proof}

\noindent Theorem \ref{theorem: positive results} states that the robust error of Algorithm \ref{algorithm: robust separated inference-time classifier} on a test point $\x$ will be small so long as its distance $r$ to its nearest training point $\x_i$ from a different class is sufficiently larger than $\tau$, and so long as the number of labeled examples $m$ is sub-exponential in $n_2-n_3$.  If $m$ is so large that a sphere of radius $r$ about point $\x$ can be covered by $m$ balls of radius $\tau$, then the adversary could indeed win, because any ray extending from $\x$ will pierce one of these balls.   One simple way to address this would be that if size of the labeled sample really is exponential in $n_2-n_3$, then to just use a sub-exponentially large random subsample of it.\footnote{This observation shows that nearest-neighbor is not an optimal algorithm when the number of examples is exponential in the dimension.  In the case of very large $m$, one could instead use an algorithm that estimated densities in each part of space.} We also prove the following  asymptotic improvement over Theorem \ref{theorem: positive results} for fixed $n_3$ and large $n_2$ via a tighter bound on the probability mass of the region of adversarial success.

\begin{restatable}{theorem}{thmimprovedbound}\label{theorem: improved bound}
If $\tau=o(r)$, the robust error $\cE_\adv^{\x}(f)$ of $\textsc{RobustClassifier}(\tau,0)$ in Algorithm \ref{algorithm: robust separated inference-time classifier} for classifying $\x$ is at most $\mathcal{O}\left(\frac{m}{n_2-n_3}\left(\frac{\tau}{r}\right)^{n_2-n_3}\frac{1}{B(n_3/2,(n_2-n_3)/{2})}\right)$, where $B(\cdot,\cdot)$ is the Beta function. \blue{The Beta function is given by $B(r_1,r_2)=\int_0^1t^{r_1-1}(1-t)^{r_2-1}dt$, for $r_1,r_2\in\R^+$, and is closely related to binomial coefficients.}
\end{restatable}

\begin{proof}
We drop $F(\cdot)$ from the notation for simplicity. Let $\x$ be the origin. Let $\x'$ be a training point of another class, and $R$ be a random $n_3$-dimensional linear subspace. Scale all distances by a factor of $\frac{1}{\dist(\x,\x')}$. By rotational symmetry, we assume WLOG that $R$ is given by $x_{n_3+1}=x_{n_3+2}=\dots=x_{n_2}=0$, and $\x'$ is the uniformly random unit vector $(z_1,\dots,z_{n_2})$. Indeed, for a fixed direction from $\x$, the set of subspaces for which the projection of $\x'$ lies along that direction is constrained by one vector each in the range space and kernel space, and is therefore in bijection to the set of subspaces associated with another fixed direction (Figure \ref{fig:rotational symmetry}).

\begin{figure}[t]
\center

\tikzset{every picture/.style={line width=0.75pt}} 

\begin{tikzpicture}[x=0.75pt,y=0.75pt,yscale=-1,xscale=1]

\draw  [dash pattern={on 0.84pt off 2.51pt}] (322.5,323.69) .. controls (326.77,321.82) and (327.97,318.74) .. (326.1,314.47) -- (320.43,301.53) .. controls (317.75,295.43) and (318.55,291.44) .. (322.82,289.56) .. controls (318.55,291.44) and (315.07,289.33) .. (312.4,283.22)(313.6,285.97) -- (306.72,270.29) .. controls (304.85,266.02) and (301.77,264.82) .. (297.5,266.69) ;
\draw    (114,339) -- (325.79,238.97) ;
\draw [shift={(328.5,237.69)}, rotate = 514.72] [fill={rgb, 255:red, 0; green, 0; blue, 0 }  ][line width=0.08]  [draw opacity=0] (8.93,-4.29) -- (0,0) -- (8.93,4.29) -- cycle    ;
\draw [shift={(114,339)}, rotate = 334.72] [color={rgb, 255:red, 0; green, 0; blue, 0 }  ][fill={rgb, 255:red, 0; green, 0; blue, 0 }  ][line width=0.75]      (0, 0) circle [x radius= 3.35, y radius= 3.35]   ;
\draw    (114,339) -- (320.5,337.69) ;
\draw [shift={(320.5,337.69)}, rotate = 359.64] [color={rgb, 255:red, 0; green, 0; blue, 0 }  ][fill={rgb, 255:red, 0; green, 0; blue, 0 }  ][line width=0.75]      (0, 0) circle [x radius= 3.35, y radius= 3.35]   ;
\draw    (291.5,256.69) -- (320.5,337.69) ;
\draw [shift={(320.5,337.69)}, rotate = 70.3] [color={rgb, 255:red, 0; green, 0; blue, 0 }  ][fill={rgb, 255:red, 0; green, 0; blue, 0 }  ][line width=0.75]      (0, 0) circle [x radius= 3.35, y radius= 3.35]   ;
\draw [shift={(291.5,256.69)}, rotate = 70.3] [color={rgb, 255:red, 0; green, 0; blue, 0 }  ][fill={rgb, 255:red, 0; green, 0; blue, 0 }  ][line width=0.75]      (0, 0) circle [x radius= 3.35, y radius= 3.35]   ;
\draw    (324.5,290.69) .. controls (337.37,297.62) and (348.28,302.59) .. (387.31,273.58) ;
\draw [shift={(388.5,272.69)}, rotate = 503.13] [color={rgb, 255:red, 0; green, 0; blue, 0 }  ][line width=0.75]    (10.93,-3.29) .. controls (6.95,-1.4) and (3.31,-0.3) .. (0,0) .. controls (3.31,0.3) and (6.95,1.4) .. (10.93,3.29)   ;
\draw  [dash pattern={on 0.84pt off 2.51pt}] (258.5,260.69) .. controls (256.43,256.51) and (253.3,255.46) .. (249.12,257.53) -- (195.69,284.03) .. controls (189.72,286.99) and (185.69,286.38) .. (183.62,282.2) .. controls (185.69,286.38) and (183.74,289.95) .. (177.77,292.91)(180.46,291.58) -- (138.66,312.31) .. controls (134.48,314.38) and (133.43,317.51) .. (135.5,321.69) ;
\draw    (186.5,274.69) .. controls (166.24,285.3) and (164.59,271.71) .. (161.81,254.57) ;
\draw [shift={(161.5,252.69)}, rotate = 440.54] [color={rgb, 255:red, 0; green, 0; blue, 0 }  ][line width=0.75]    (10.93,-3.29) .. controls (6.95,-1.4) and (3.31,-0.3) .. (0,0) .. controls (3.31,0.3) and (6.95,1.4) .. (10.93,3.29)   ;
\draw    (274.5,266.69) .. controls (283.5,289.69) and (272.5,284.69) .. (295.5,275.69) ;

\draw (90,338) node [anchor=north west][inner sep=0.75pt]   [align=left] {$\displaystyle \x$};
\draw (331,337) node [anchor=north west][inner sep=0.75pt]   [align=left] {$\displaystyle \x'$};
\draw (368,253) node [anchor=north west][inner sep=0.75pt]   [align=left] {in kernel space};
\draw (118,235) node [anchor=north west][inner sep=0.75pt]   [align=left] {in range space};
\draw (331,216) node [anchor=north west][inner sep=0.75pt]   [align=left] {$\displaystyle \y$};
\draw (264,231) node [anchor=north west][inner sep=0.75pt]   [align=left] {Proj$\displaystyle [ \x']$};

\end{tikzpicture}
\caption{Rotational symmetry of adversarial subspaces. Let $\y$ be a random direction from test point $\x$, and Proj$[\x']$ be the projection of training point $\x'$ on to $\x\y$. For any adversarial space with Proj$[\x']$ as the projection of $\x'$ on the space, we must have $\x\y$ in the range space and $\x'$Proj$[\x']$ in the nullspace.\label{fig:rotational symmetry}}
\end{figure}

The adversary can win only if the distance between $\x'$ with the closest vector Proj$[\x']$ in $R$, that is with $(z_1,\dots,z_{n_3},0,\dots,0)$, is at most $\frac{\tau}{\dist(\x,\x')}\le \frac{\tau}{r}$. We can now apply Lemma \ref{lem:hsphere} (Appendix \ref{sec:upper bound beta function}), which gives a bound on the fraction of the surface of the sphere at some fixed small distance from the orthogonal space, to get that the adversary succeeds by perturbing $\x$ to a point within $\cB(\x',\tau)$ with probability at most
$$\frac{2(\tau/r)^{n_2-n_3}}{n_2-n_3}\cdot\frac{A(n_2-n_3-1)A(n_3-1)}{A(n_2-1)}, $$
where $A(n)$ is the surface-area of the unit $n$-sphere embedded in $\R^{n+1}$. We have a closed form $A(n)=\frac{2\pi^{\frac{n+1}{2}}}{\Gamma\left(\frac{n+1}{2}\right)}$, where $\Gamma(z)=\int_{t=0}^\infty t^{z-1}e^{-t}dt$ is the gamma function.

Noting that $B(z_1,z_2)=\frac{\Gamma(z_1)\Gamma(z_2)}{\Gamma(z_1+z_2)}$, together with a union bound over all training points from a different class, gives the result. 
\end{proof}

\subsection{Outlier Removal and Improved Upper Bound}
The guarantees above are good when the test points are far from training points from other classes in the feature space. This empirically holds true for good data and perfect embeddings---a so-called neural collapse phenomenon that the trained network converges to representations such that all points of class $k$ get embedded close to a single point $\mu_k$ in the feature space~\citep{papyan2020prevalence}. But for noisy data and good-but-not-perfect embeddings, the condition may not hold. In Theorem \ref{theorem: robust accuracy with separation} (in Appendix \ref{app:outlier}) we show that we obtain almost the same upper bound on failure probability as above by exploiting the outlier removal threshold $\sigma$. Intuitively, outlier removal artificially induces well-separateness in the feature space, by deleting training examples that are close to other examples with a different label.

\subsection{Upper Bound on Abstention Rate on Natural Data}

Of course, the statement that robust error is low just means the adversary has a low probability of being able to create an error.  This is only half the picture: the other half is that we also want our algorithm to have a low probability of abstaining on natural data.  This is what we address in the next two sections, and it will require assumptions on how natural data is distributed.  In particular, we give two different sufficient conditions for having a low abstention rate on natural data: (1) that natural data is well-clustered in feature space (Section \ref{sec:well-clustered}), and alternatively (2) that the natural data has low {\em doubling dimension} (Section \ref{sec:doubling}).  For these results, we assume our $m$ training points $\x_1,\ldots,\x_m$ are i.i.d.~draws from distribution $\cD_{F(\cX)}$; if we also have additional training points used in the construction of $F$ (which, therefore, cannot be treated as i.i.d.~draws), this can only help.

\subsubsection{Low Abstention Rate for Well-clustered Data}
\label{sec:well-clustered}
We show here that if natural data has the property that for every label class, one can cover most of the probability mass of the class with not too many (potentially overlapping) balls of at least some minimal probability mass, then our algorithm will have a low abstention rate.  

\blue{
\begin{definition} 
\label{assumption: covering}
A distribution $\cD$ is $(\delta,\beta,N)$-coverable if at least a $1-\delta$ fraction of probability mass of the marginal distribution $\cD_{F(\cX)}$ over $\R^{n_2}$ can be covered by $N$ balls $\bbB_1$, $\bbB_2$, ... $\bbB_N$ of radius $\tau/2$ and of mass $\Pr_{\cD_{F(\cX)}}[\bbB_k]\ge \beta$.
\end{definition}

\noindent Intuitively, if a set of balls cover (most of) the distribution and we sample enough points from the distribution, we should get at least one sample from each ball and our algorithm will not abstain on the covered points. 
Formally, we show the following guarantee on the abstention rate on distributions that are $(\delta,\beta,N)$-coverable w.r.t. threshold $\tau$.

\begin{theorem}
\label{theorem: samples nearly cover all support}
Suppose that $F(\x_1),...,F(\x_m)$ are $m$ training instances i.i.d. sampled from marginal distribution $\cD_{F(\cX)}$. If the distribution $\cD$ is $(\delta,\beta,N)$-coverable, 
for sufficiently large $m\ge\frac{1}{\beta}\ln\frac{N}{\gamma}$, 
with probability at least $1-\gamma$ over the sampling, we have $\Pr[\cup_{i=1}^m \bbB(F(\x_i),\tau)]\ge 1-\delta$.
\end{theorem}

\begin{proof} Fix ball $\bbB_i$ in the cover from Definition \ref{assumption: covering}. Let $B_i$ denote the event that no point is drawn from ball $\bbB_i$ over the $m$ samples. Since successive draws are independent, and by Definition \ref{assumption: covering} $\Pr_{\cD_{F(\cX)}}[\bbB_i]\ge {\beta}$, we have that $\Pr[B_i]\le \left(1-{\beta}\right)^m\le \exp(-\beta m)$. Further, by a union bound over $N$ balls $\Pr[\cup_iB_i]\le N\exp(-\beta m)\le \gamma$, for $m\ge\frac{1}{\beta}\ln\frac{N}{\gamma}$.

Therefore, with probability at least $1-\gamma$ for all $k\in[N]$ there is at least a sample $F(\x_{i_k})\in\{F(\x_1),F(\x_2),...,F(\x_m)\}$ such that $F(\x_{i_k})\in\bbB_k$. This implies $\cup_{i=1}^m\bbB(F(\x_i),\tau)\supseteq \cup_{k=1}^N\bbB_k$, since $\bbB_k$ is a ball of radius $\tau/2$. So with probability at least $1-\gamma$ over the sampling, we have $\Pr[\cup_{i=1}^m\bbB(F(\x_i),\tau)]\ge \Pr[\cup_{k=1}^N\bbB_k]\ge 1-\delta.$
\end{proof}
Note that in the special case that the $N$ balls are disjoint and each has probability mass $\beta=1/N$, then $m = \Omega(N\log N)$ samples are also necessary to get a point inside each ball, by a standard coupon-collector analysis.

Theorem \ref{theorem: samples nearly cover all support} implies that if we have a covering with $N$ balls, each with probability mass at least ${\beta}{}$ and large enough sample size $m$, with probability at least $1-\gamma$ over the sampling, we have $\Pr[\cup_{i=1}^m \bbB(F(\x_i),\tau)]\ge 1-\delta$. Therefore, with high probability, the algorithm will output ``don't know'' only for a $\delta$ fraction of natural data.  Below we give an example of a distribution where our algorithm will simultaneously achieve low robust error and low natural abstention rates.

}

\smallskip
\blue{
\noindent\textit{Example distribution where Algorithm \ref{algorithm: robust separated inference-time classifier} is robust with low abstention rate.} Our example will consist of well-separated data in the feature space. Suppose $\cD_{F(\cX)|y}$ for each label class $y$ consists of the uniform distribution over $N_y$ $n_2$-balls of radius $\tau/2$ centered at axis-aligned unit vectors $\{\mathbf{e}_j\mid j\in S_y\}$, where $S_y\subset [n_2]$ is the set of axes with balls labeled by $y$, with $\tau<1/3$ and $S_{y}\cap S_{y'}=\emptyset$ for $y\ne y'$. Further let $m=n_2\log \frac{n_2}{\gamma}$ for some absolute constant $\gamma\in(0,1)$, so this distribution is $(\delta,\beta,N)$-coverable with $\delta=0$, $\beta=1/N$ and $N=n_2$.   If $n_3=1$, by Theorem \ref{theorem: positive results}, the robust error of Algorithm \ref{algorithm: robust separated inference-time classifier} is bounded by $O(n_2\log n_2\tau^{n_2-1}) = o(1)$. 
Thus, in this setting, our algorithm  enjoys low robust error without abstaining too much (for sufficiently large $n_2$).
}

\subsubsection{Controlling Abstention Rate via Doubling Dimension}
\label{sec:doubling}
Here, we give an alternative bound on the abstention rate on natural data based on the {\em doubling dimension} of the data distribution. Doubling dimension can be used to obtain sample complexity of generalization for learning problems \citep{bshouty2009using}. \blue{Bounded doubling dimension has also been used to give bounds on cluster quality for nearest-neighbor based extensions of clustering algorithms in the distributed learning setting \citep{dick2017data}. }  

\begin{definition}[Doubling dimension]
A measure $\cD_{F(\cX)}$ with support $F(\cX)$ is said to have a doubling dimension $d$, if for all points $F(\x)\in F(\cX)$ and all radii $\tau>0$, $\cD_{F(\cX)}(\cB(F(\x),2\tau))\le 2^d\cD_{F(\cX)}(\cB(F(\x),\tau))$.
\end{definition}

\noindent Given a sample $S$ and point $\x$, let $NN_S(F(\x))$ denote $\x$'s nearest neighbor in $S$ in feature space.  We now have the following theorem, which implies low abstention under the assumption of bounded doubling dimension, which we show implies that $\cD$ satisfies Definition \ref{assumption: covering} (for $\beta,N$ that depend on the doubling dimension).
\blue{
\begin{theorem}\label{thm:dd}
Suppose that the measure $\cD_{F(\cX)}$ in the feature space has a doubling dimension $d$. Let $D$ be the diameter of $F(\cX)$.
For any $\tau>0$ and any $\delta>0$, if we draw an i.i.d. sample $S$ of size $m\ge \left(\frac{2D}{\tau}\right)^{d}\left(d\log\frac{4D}{\tau}+\log\frac{1}{\delta}\right)$, then with probability at least $1-\delta$ over the draw of $S$, we have $\sup_{\x\in \cX}d(F(\x),NN_S(F(\x)))\le \tau$.
\end{theorem}

\begin{proof}
Lemma \ref{lemma: relating doubling dimension to covering number} (below) implies that there exists a covering of $F(\cX)$ of size $(4D/\tau)^d$ which consists of balls of radius $\tau/2$ around points $F(\x) \in F(\cX)$. Further Lemma \ref{lemma: probability of ball} implies that for a ball $B$ of radius $\tau/2$ around point $F(\x) \in F(\cX)$ we have $\cD_{F(\cX)}(B)\ge \left(\frac{\tau}{2D}\right)^{d}$. Thus Definition \ref{assumption: covering} is satisfied with $N=(4D/\tau)^d$ and $\beta=\left(\frac{\tau}{2D}\right)^{d}$. Theorem \ref{theorem: samples nearly cover all support} now implies the result.
\end{proof}
}

\noindent\blue{Our proof of Theorem \ref{thm:dd} relies on the following properties of a doubling measure. We first give a lower bound on the probability mass of a small ball in terms of the doubling dimension of the distribution.}

\begin{lemma}
\label{lemma: probability of ball}
Suppose that the measure $\cD_{F(\cX)}$ has a doubling dimension $d$. Let $D$ be the diameter of $F(\cX)$.
Then for any point $F(\x)\in F(\cX)$ and any radius of the form $\tau=D/2^T$ for $T\in\mathbb{N}$, we have $\cD_{F(\cX)}(\cB(F(\x),\tau))\ge (\tau/D)^d$.
\end{lemma}

\begin{proof}
Since $D$ is the diameter of $F(\cX)$, we have $\cD_{F(\cX)}(\cB(F(\x),D))=1$. Therefore, we have
\begin{equation*}
\begin{split}
\cD_{F(\cX)}(\cB(F(\x),\tau))&=\cD_{F(\cX)}(\cB(F(\x),D/2^T))\\
&\ge 2^{-d}\cdot \cD_{F(\cX)}(\cB(F(\x),D/2^{T-1}))\\
&\ge \cdots\\
&\ge 2^{-Td}\cdot \cD_{F(\cX)}(\cB(F(\x),D))\\
&=2^{-Td}\\
&=(\tau/D)^d.
\end{split}
\end{equation*}
\end{proof}

\noindent\blue{This further lets us bound the covering number in terms of the doubling dimension as follows.}

\begin{lemma}[Relating doubling dimension to covering number]
\label{lemma: relating doubling dimension to covering number}
Given any radius $\tau$ of the form $\tau=D/2^T$ for $T\in\mathbb{N}$, there is a covering of $F(\cX)$ using balls of radius $\tau$  around points $F(\x) \in F(\cX)$ of size no more than $(2D/\tau)^d$.
\end{lemma}

\begin{proof}
We construct the covering balls of $F(\cX)$ as follows: when there is a point $F(\x)\in F(\cX)$ which is not contained in any current covering ball of radius $\tau$, we add the ball $\cB(F(\x),\tau)$ to the cover. We follow this procedure until every point in $F(\cX)$ is covered by some covering balls. Denote by $\cC$ the set of centers for the balls in the cover.

We now show that this procedure stops after adding at most $(2D/\tau)^d$ balls to the cover. We note that by our construction, the centers of the covering are at least distance $\tau$ from each other, implying that the collection of $\cB(F(\x),\tau/2)$ for $F(\x)\in\cC$ are disjoint. This yields
\begin{equation*}
\begin{split}
1&\ge \cD_{F(\cX)}\left(\cup_{F(\x)\in\cC}\cB(F(\x),\tau/2)\right)\\
&=\sum_{F(\x)\in\cC} \cD_{F(\cX)}(\cB(F(\x),\tau/2))\quad\;\text{(since $\cB(F(\x),\tau/2)$ are disjoint)}\\
&\ge \sum_{F(\x)\in\cC} \left(\frac{\tau}{2D}\right)^d\qquad\qquad\qquad\quad\text{(by Lemma \ref{lemma: probability of ball})}\\
&= |\cC| \left(\frac{\tau}{2D}\right)^d.
\end{split}
\end{equation*}
So we have $|\cC|\le (2D/\tau)^d$.
\end{proof}

\subsection{A More General Adversary with Bounded Density}\label{sec:bounded-density}
We extend our results in Theorem \ref{theorem: positive results} to a more general class of adversaries, which have a bounded density over the space of linear subspaces of a fixed dimension $n_3$ and the adversary can perturb a test feature vector arbitrarily in the sampled adversarial subspace.  Specifically, a distribution is said to be {\em $\kappa$-bounded} if the corresponding probability density $f(x)$ satisfies, $\sup_xf(x)\le\kappa$. For example, the standard normal distribution $\cN(\mu,\sigma)$ is $\frac{1}{\sqrt{2\pi}\sigma}$-bounded.

\begin{restatable}{theorem}{thmboundedadversary}\label{theorem: bounded density}
Consider the setting of Theorem \ref{theorem: positive results}, with an adversary having a $\kappa$-bounded distribution over the space of linear subspaces of a fixed dimension $n_3$ for perturbing the test point. If $\E(\tau,r)$ denotes the bound on error rate in Theorem \ref{theorem: positive results} for $\textsc{RobustClassifier}(\tau,0)$ in Algorithm \ref{algorithm: robust separated inference-time classifier}, then the error bound of the same algorithm against the $\kappa$-bounded adversary is $\cO(\kappa \E(\tau,r))$.
\end{restatable}

\begin{proof}
To argue upper bounds on failure probability, we consider the set of adversarial subspaces which can allow the adversary to perturb the test point $x$ close to a training point $x'$. Let $\cS(x',\tau)$ denote the subset of linear subspaces of dimension $n_3$ such that for any $S\in\cS(x',\tau)$ there exists $v\in S$ with $x+v\in \cB(x',\tau)$. Note that 
we can upper bound the fraction of the total probability space occupied by $\cS(x',\tau)$ by $\frac{1}{m}\E(\tau,r)$, where constants in $n_2,n_3$ have been suppressed.
If we show that $\cS(x',\tau)$ is a measurable set, we can use the $\kappa$-boundedness of the adversary distribution to claim that the failure probability for misclassifying as $x'$ is upper bounded by $\kappa \vol(\cS) \frac{1}{m}\E(\tau,r)=\cO\left( \frac{\kappa}{m}\E(\tau,r)\right)$, since the volume of the complete adversarial space $\cS$ is a constant in $n_2,n_3$. In Lemma \ref{lem: convex adversary} (Appendix \ref{sec:upper bound beta function}), we make the stronger claim that $\cS(x',\tau)$ is convex. We can then use a union bound on the training points to get a bound on the total failure probability as $\cO\left(\kappa \E(\tau,r)\right)$.
\end{proof}

\section{Learning Data-Specific Optimal Thresholds}\label{sec:dd}
Given an embedding function $F$ and a classifier $f_\tau$ which outputs either a predicted class if the nearest neighbor is within distance $\tau$ of a test point or abstains from predicting if not (see Algorithm \ref{algorithm: robust separated inference-time classifier}), we want to evaluate the performance of $f_\tau$ on a test set $\cT$ against an adversary which can perturb a test feature vector in a random $n_3$-dimensional subspace $S\sim\cS$. \blue{To this end, we define
\begin{definition}[Robust error.]
Let $\cE_\adv(\tau,S):=\frac{1}{|\cT|}\sum_{(\x,\y)\in \cT}\1\{\exists \e\in S+F(\x)\subseteq \R^{n_2} \text{ such that }\allowbreak f_{\tau}(\e)\ne\y \text{ and $f_\tau(\e)$ does not abstain}\}$ denote the robust error on the test set $\cT$, for $n_3$-dimensional perturbation subspace $S$ and threshold setting $\tau$ in Algorithm \ref{algorithm: robust separated inference-time classifier}. Also define average robust error as $\cE_\adv(\tau):=\mathbb{E}_{S\sim\cS}\left[\cE_\adv(\tau,S)\right]$ for distribution $\cS$ over $n_3$-dimensional subspaces (assumed to be the uniform distribution unless stated otherwise) and estimated robust error over a set $\hat{S}$ of subspaces as $\hat{\cE}_\adv(\tau,\hat{\cS}):=\frac{1}{|\hat{\cS}|}\sum_{S\in \hat{\cS}}\cE_\adv(\tau,S)$. Let $\hat{\cS}$ consist of multiple samples drawn from $\cS$, and for conciseness, we will often denote  $\hat{\cE}_\adv(\tau,\hat{\cS})$ by $\hat{\cE}_\adv(\tau)$ and $\hat{\cS}$ will be implicit from context.
\end{definition}

\noindent $\hat{\cE}_\adv(\tau)$ gives an easier-to-compute surrogate to $\cE_\adv(\tau)$, by drawing subspaces in $\hat{\cS}$ according to $\cS$ (Algorithm \ref{algorithm: adversarial attack opt} gives the procedure to compute the attack perturbation given subspace $S$). For an abstentive classifier, the robust error can be trivially minimized by abstaining everywhere. We will therefore also need to control the abstention rate on unperturbed data.

\begin{definition}[Natural abstention rate.] Define
$\cD_\nat(\tau):=\frac{1}{|\cT|}\sum_{(\x,\y)\in \cT}\1\{f_\tau(F(\x))\text{ abstains}\}$ as the abstention rate on the unperturbed test set $\cT$.
\end{definition}

 \noindent$\cE_\adv(\tau)$ and $\cD_\nat(\tau)$ are both monotonic in $\tau$; while the former is non-decreasing, the latter is non-increasing (Lemma \ref{lem:monotonicity}). 
 
 \begin{lemma}\label{lem:monotonicity}
  Robust error $\cE_\adv(\tau,S)$ is monotonically non-decreasing in $\tau$ for any $S$. Further, natural abstention rate $\cD_\nat(\tau)$ is monotonically non-increasing in $\tau$.
 \end{lemma}
 \begin{proof}
 Let $0\le\tau_1\le\tau_2\le \infty$. For any $(\x,\y)\in\cT$, if there exists $\e\in S+F(\x)$ for which the adversary succeeds for threshold $\tau_1$, we have $f_{\tau_1}(\e)\ne\y$ and $f_{\tau_1}(\e)$ does not abstain. Since $f_{\tau_2}$ does not abstain whenever $f_{\tau_1}$ does not abstain, we have in particular that $f_{\tau_2}(\e)$ does not abstain. Moreover, conditioned on not abstaining, we have $f_{\tau_2}(\e)=f_{\tau_1}(\e)\ne\y$. Thus $\cE_\adv(\tau_2,S)$ incurs error for each test point $(\x,\y)$ for which $\cE_\adv(\tau_2,S)$ incurs an error, implying monotonicity in $\tau$. A similar argument for counting the abstention on any fixed test point for any pair of values of the threshold implies $\cD_\nat(\tau)$ is monotonically non-increasing.
 \end{proof}
 
 \noindent Lemma \ref{lem:monotonicity} further implies that $\cE_\adv(\tau)$ and $\hat{\cE}_\adv(\tau)$ are also monotonic non-decreasing in $\tau$. The robust error $\cE_\adv(\tau)$ is optimal at $\tau=0$, but this implies that we abstain from prediction all the time (that is, $\cD_\nat(0)=1$). Conversely, we can minimize the abstention rate by not abstaining, that is, $\cD_\nat(\infty)=0$ corresponding to vanilla nearest-neighbor, but this maximizes the robust error. This motivates us to consider the following objective function which combines robust error and natural abstention rate. 

\begin{definition}[Robust Chow's objective.]\label{def:robust-chow}
Define $l(\tau):=\cE_\adv(\tau)+c \cD_\nat(\tau)$ as the robust Chow's objective, where $c$ is a positive constant and denotes the {\it cost of abstention}. Further define $\hat{l}(\tau):=\hat{\cE}_\adv(\tau)+c \cD_\nat(\tau)$ as the estimated robust Chow's objective.
\end{definition}

\noindent Definition \ref{def:robust-chow} may be viewed as an adversarial version of Chow's objective for abstentive classifiers \citep{chow1970optimum}, which uses natural risk instead of adversarial risk. If, for example, we are willing to take a one percent increase of the abstention rate for a two percent drop in the error rate, we could set $c$ to  $\frac{1}{2}$. For a single test set $\cT$, the abstention rate $\cD_\nat(\tau)$ can change at (at most) $|\cT|$ `critical' values of $\tau$ corresponding to nearest neighbor distances. Given oracle access to $\cE_\adv(\tau)$, we can minimize $l(\tau)$ over the given test sample with at most $|\cT|$ evaluations. Suppose, however, the test data arrives sequentially in batches of size $b$, potentially from related tasks with different data distributions, and we need to figure out how to set the threshold $\tau$ for unseen tasks. As we will show, techniques from data-driven algorithm design \citep{balcan2018dispersion,balcan2021learning} can help approach this multi-task robustness setting.

Formally, we define our online learning setting as follows. Consider a game consisting of $T$ rounds. In each round $t=1,\dots,T$, the learner is presented with a new test batch $\cT_t$ of size $b$. In Theorem \ref{thm:regret}, we show no regret can be achieved for online learning of the threshold $\tau$ using test batches of size $b$ (consisting of unperturbed points) on which the learner chooses abstention threshold $\tau_t$, that is, predicting using classifier $f_{\tau_t}$. Let $l_t$ (resp. $\hat{l}_t$) be the (resp. estimated) robust Chow's objective on the test set $\cT_t$. The learner suffers loss $l_t(\tau_t)$ and observes $l_t(\tau)$. The goal of the learner is to minimize total expected regret, defined as $R_T:=\mathbb{E}\left[\sum_{t=1}^Tl_t(\tau_t)-\min_{\tau}\sum_{t=1}^Tl_t(\tau)\right]$, where the expectation is over the randomness of the loss functions as well as learner's internal randomness.

Our main result is the following theorem (Theorem \ref{thm:regret}) in the above setting. Our proof strategy is to show that the sequence of loss functions $l_t(\tau)$ is $(w,k)$-{\it dispersed} in the sense of \cite{balcan2018dispersion}. We present a simplified definition of \emph{dispersion} for real-valued functions.

\begin{definition}[Dispersion, \cite{balcan2018dispersion}]
Let $u_1,\dots,u_T :\R\rightarrow [0,1]$ be a collection of functions where $u_i$ is piecewise Lipschitz over a partition $P_i$ of $\R$. We say that $P_i$ splits a set $A$ if $A$ intersects with at least two sets in $P_i$. The collection of functions is $(w,k)$-dispersed if every interval of length $w$ is split by at most $k$ of the partitions $P_1, \dots , P_T$.\label{def:dispersion}
\end{definition}

\begin{algorithm}[t]
\caption{Exponential Forecaster Algorithm \citep{balcan2018dispersion}}
\label{alg:ef}
\begin{algorithmic}[1]
\STATE {\bfseries Input:} step size parameter $\lambda \in (0, 1]$.
\STATE {\bfseries Output:} thresholds $\tau_t$ for times $t=1,2,\dots,T$.
\STATE{Set $w_1(\tau)=1$ for all $\tau\in [0,\tau_{\max}]$.}
\FOR{$t=1,2,\dots,T$}
\STATE{$W_t:=\int_{[0,\tau_{\max}]}w_t(\tau)d\tau$.}
\STATE{Sample $\tau$ with probability proportional to $w_t(\tau)$, that is, with probability
        $p_{t}(\tau)=\frac{w_t(\tau)}{W_t}$. Output the sampled $\tau$ as $\tau_t$.}
\STATE{Observe $l_t(\cdot)$. Set $u_t(\tau):=1-\frac{l_t(\tau)}{1+c}$.}
\STATE{For each $\tau\in [0,\tau_{\max}], \text{ set }w_{t+1}(\tau)=e^{\lambda u_t(\tau)}w_{t}(\tau)$.}
\ENDFOR
\end{algorithmic}
\end{algorithm}

\noindent Intuitively, if a sequence of functions is piecewise-Lipschitz except for a finite number of breakpoints (or points of discontinuity), it is said to be dispersed if the discontinuities do not concentrate in a small region of the domain space over time. Finally, we will employ known results about no-regret learning of $(w,k)$-dispersed functions using Algorithm \ref{alg:ef} a continuous version of Exponential Weights algorithm for finite experts \citep{balcan2018dispersion}. Proofs of the technical lemmas needed for proving Theorem \ref{thm:regret} can be found in Appendix \ref{app:dd}. }

\begin{restatable}{theorem}{thmdispersionbound}\label{thm:regret}
Consider the online learning setting described above. Assume $\tau\in[0,\tau_{\max}]$ with  $\tau_{\max}=o\left(r\right)$, $r>1$, and each test batch $\cT_t$ is sampled from a data distribution $\cD$ that has $\kappa$-bounded density. If $\tau_t$ is set using a continuous version of the multiplicative updates algorithm, Algorithm \ref{alg:ef}, for $T$ rounds of the online game, then with probability at least $1-\delta$, the total expected regret of the learner for the loss sequence $l_t(\tau)$ is bounded by $O\left(\sqrt{n_2T\log\left(\frac{\kappa  mb\tau_{\max}T}{\delta}\right)}\right)$. Here $b$ is the batch size and $r$ is the smallest distance between points of different labels.
\end{restatable}
\begin{proof} \blue{We show the sequence of loss functions $l_t(\tau)$ is $(w,k)$-{\it dispersed} (Definition \ref{def:dispersion}) in two steps. We first argue that the robust error part of the loss $l(\tau)$ is Lipschitz, and we further show that the natural abstention rate is piecewise constant with dispersed discontinuities.
}

A key challenge is to analyze the adversary success probability and show  that $\cE_\adv(\tau)$ is Lipschitz for sufficiently small $\tau$. In Lemma \ref{lem:acc-lipschitzness} (see  Appendix \ref{app:dd} for a proof), we show that $\cE_\adv(\tau)$ is $L$-Lipschitz, where $L=O\left(m\tau_{\max}^{n_2-n_3-1}/r^{n_2-n_3}\right)$. \blue{Intuitively, for any test point the probability the adversary succeeds by perturbing to within a distance $\tau$ and $\tau+\Delta$ of a fixed training point can be upper bounded using arguments similar to our proof of bounds on robust error in Section \ref{sec: positive results}. A union bound over training points then gives the bound on $L$.  Note that $\cD_\nat(\tau)$ is piecewise constant. This is because, for any set $\cT_t$ of test points, we have at most $|\cT_t|$ points corresponding to distances of the test points to the nearest training point, where the function value decreases by $\frac{1}{|\cT_t|}$. Together with $L$-Lipschitzness of $\cE_\adv(\tau)$, this implies $l(\tau)$ is piecewise $L$-Lipschitz. 

In Lemma \ref{lem:dk-dispersion} we show that, for batch size $b$, $\cD_\nat(\tau)$ has $O(\kappa bmw \tau_{\max}^{n_2-1})$ discontinuities in expectation (over the data distribution) in any interval of width $w$. Note that if a discontinuity occurs within the interval $I=[\tau,\tau+w]$, then there must exist a test point $\x$ in the test set $\cT$ for which the nearest-neighbor training point is at distance $\tau'\in I$. That is, the training point lies within $\cB(\x,\tau+w)\setminus \cB(\x,\tau)$. The proof involves bounding the fraction of points at distance $d\in[\tau,\tau+w]$ for any test point, using smoothness of the data distribution, and using a union bound over the $b$ test points. See Appendix \ref{app:dd} for a formal argument. Since $\cE_\adv(\tau)$ is Lipschitz continuous, $l(\tau)$ has at most $O\left(\kappa bmw \tau_{\max}^{n_2-1}\right)$ discontinuities in expectation in any $w$-interval. 

Using a standard argument based on the VC-dimension of 1D intervals (for example, Theorem 7 in \cite{balcan2020semi}), the maximum number of discontinuities in any interval of width $w$ is $k=O\left(\kappa bmw \tau_{\max}^{n_2-1}T+\sqrt{T\log\frac{b}{\gamma}}\right)$ with high probability $1-\gamma$. In other words, $l(\tau)$ is $(w,k)$-Lipschitz with high probability over the data distribution. This allows us to use a continuous version of standard Exponential Weights update introduced by \cite{balcan2018dispersion} as our online algorithm (which we include as Algorithm \ref{alg:ef} for completeness), for which they show an $O\left(\sqrt{T\log\frac{R}{w}}+k+wLT\right)$ bound on the expected regret if the sequence of loss functions is $(w,k)$-dispersed with $L$-Lipschitz pieces, where $R$ is a bound on the diameter of the continuous domain ($R=\tau_{\max}$ in our setting).

} \noindent Formally, we can apply Theorem \ref{thm:dispersion} with $w=\frac{1}{\kappa bm \tau_{\max}^{n_2-1}\sqrt{T}}$ to get the desired regret bound.
\begin{align*}
    R_T  &= O\left(\sqrt{T\log\frac{R}{w}}+k+wLT\right)\\
    &\le O\left(\sqrt{T\log\frac{\tau_{\max}}{(\kappa bm \tau_{\max}^{n_2-1}\sqrt{T})^{-1}}}+O\left(\sqrt{T}+\sqrt{T\log\frac{b}{\delta}}\right)+\frac{O(m\tau_{\max}^{n_2-n_3-1}/r^{n_2-n_3})}{\kappa bm \tau_{\max}^{n_2-1}\sqrt{T}}\cdot T\right)\\
    &\le O\left(\sqrt{T\log\left(\frac{\kappa  mb\tau_{\max}^{n_2}T}{\delta}\right)}\right),
\end{align*}
where the first inequality holds with probability at least $1-\delta$.
\end{proof}

\noindent\blue{A similar no-regret learning guarantee can also be given for the estimated robust Chow's objective $\hat{l}(\tau)$. In practice $l(\tau)$ can be hard to compute, but as discussed above the learner can more easily estimate this loss by computing $\hat{l}(\tau)$. The key difference in the proof is that the estimated robust error $\hat{\cE}_\adv(\tau)$ is piecewise constant, while $\cE_\adv(\tau)$ was shown to be Lipschitz for small $\tau$. Roughly speaking, we will use smoothness of the adversary distribution to argue that location of discontinuities of $\hat{\cE}_\adv(\tau)$ cannot concentrate in a small interval. Formally, we show that

\begin{restatable}{theorem}{thmdispersionbound-lhat}\label{thm:regret-lhat}
Consider the online learning setting described above. Assume $\tau\in[0,\tau_{\max}]$ with  $\tau_{\max}=o\left(r\right)$, $r>1$ and each test batch $\cT_t$ is sampled from a data distribution $\cD$ that has $\kappa$-bounded density. If $\tau_t$ is set using a continuous version of the multiplicative updates algorithm, Algorithm \ref{alg:ef}, for $T$ rounds of the online game, then with probability at least $1-\delta$, the total expected regret of the learner for the loss sequence $\hat{l}_t(\tau)$ is bounded by $O\left(\sqrt{n_2T\log\left(\frac{\kappa  msb\tau_{\max}T}{\delta}\right)}\right)$.
Here $b$ is the batch size, $s$ is the number of sample subspaces used to estimate the robust Chow's objective $\hat{l}(\cdot)$ and $r$ is the smallest distance between points of different labels.
\end{restatable}
\begin{proof} 
Lipschitzness of $\cE_\adv(\tau)$ also implies that the breakpoints of $\hat{\cE}_\adv(\tau)$ are smoothly distributed, in particular in any interval of width $w$, we have at most $O\left(bmw\tau_{\max}^{n_2-n_3-1}/r^{n_2-n_3}\right)$ discontinuities (Corollary \ref{cor:estimated-loss}), in expectation over the draw of the adversarial subspace. The rest of the argument is very similar to part (i) above.
\end{proof}
}

\blue{

\noindent In this work we restrict our attention to the {\it full information} setting where entire function $l_t(\tau)$ is available to the learner after the prediction in round $t$. It is an interesting future question to model the adversary with bandit feedback where only $l_t(\tau_t)$ is revealed to the learner.  
The test sets $\cT_t$ may be adversarial as long as they are generated by smooth but possibly different data distributions (in the sense of Theorem \ref{thm:regret}). Our experiments in Section \ref{sec:expt} indicate Algorithm \ref{algorithm: robust separated inference-time classifier} can be made more effective by tuning both parameters $\tau$ and $\sigma$ together. Effective tuning of data-driven algorithms with multiple parameters is an interesting research direction \citep{balcan2022faster}.
Finally, we perform the analysis for tuning our relatively simple thresholded nearest-neighbor approach, but  data-driven algorithm design may prove useful for selecting the best data-specific robust approach from candidate algorithms more generally.

\begin{remark}
A simple goal for setting $\tau$ is to fix an upper limit $d^*$ on $\cD_\nat(\tau)$, corresponding to a maximum abstention rate allowed on the natural data. It is straightforward to search for an optimal $\tau^*$ such that $\cD_\nat(\tau^*)= \max_{\tau,\cD_\nat(\tau)\le d^*} \cD_\nat(\tau)$---simply use the nearest neighbor distances (to training examples) for the test points to compute the abstention rate at any $\tau$, and do a binary search for $d^*$. For $\tau<\tau^*$ we have a higher abstention rate, and when $\tau>\tau^*$ we have a higher robust error rate. 
For more sophisticated goals, for example minimizing objectives that depend on both $\cE_\adv(\tau)$ and $\cD_\nat(\tau)$, we may not be able to perform a binary search, though a linear search would still suffice. 
Here we have considered a setting where we have multiple test sets, conceptually coming from different but related tasks in some domain, and rather than separately performing this parameter tuning on each task, we want instead to learn a common value of $\tau$ that works well across all the tasks. 
\end{remark}
}

\subsection{A simple intuitive example with exact calculation demonstrating significance of data-driven algorithm design}
\label{section: toy example}
The significance of data-driven design in this setting is underlined by the following two observations. Firstly, as noted above, optimization for $\tau$ across problem instances is difficult due to the non-Lipschitz nature of $\cD_\nat(\tau)$ and the intractability of characterizing the objective function $l(\tau)$ exactly due to $\cE_\adv(\tau)$. Secondly, the optimal value of $\tau$ can be a complex function of the data geometry and sampling rate. We illustrate this by exact computation of optimal $\tau$ for a simple intuitive setting: consider a binary classification problem where the features lie uniformly on two one-dimensional manifolds embedded in two-dimensions (that is, $n_2=2$, see Figure \ref{figure: optimal threshold toy example}). Assume that the adversary perturbs in a uniformly random direction ($n_3=1$).
Further assume that our training set consists of $2m$ examples, $m$ from each class. In this toy setting, we  show that the optimal threshold varies with data-specific factors.

\begin{figure}[ht]
\center
\tikzset{every picture/.style={line width=0.75pt}} 

\begin{tikzpicture}[x=0.75pt,y=0.75pt,yscale=-1,xscale=1]

\draw [color={rgb, 255:red, 45; green, 19; blue, 241 }  ,draw opacity=1 ][line width=2.25]    (485.5,69.69) -- (421.5,69.69) ;
\draw    (118.5,49.69) -- (178.5,49.69) ;
\draw [shift={(180.5,49.69)}, rotate = 180] [color={rgb, 255:red, 0; green, 0; blue, 0 }  ][line width=0.75]    (10.93,-3.29) .. controls (6.95,-1.4) and (3.31,-0.3) .. (0,0) .. controls (3.31,0.3) and (6.95,1.4) .. (10.93,3.29)   ;
\draw [shift={(116.5,49.69)}, rotate = 0] [color={rgb, 255:red, 0; green, 0; blue, 0 }  ][line width=0.75]    (10.93,-3.29) .. controls (6.95,-1.4) and (3.31,-0.3) .. (0,0) .. controls (3.31,0.3) and (6.95,1.4) .. (10.93,3.29)   ;
\draw    (190.5,69.19) -- (412.5,69.19) ;
\draw [shift={(418.5,69.19)}, rotate = 180] [color={rgb, 255:red, 0; green, 0; blue, 0 }  ][line width=0.75]    (10.93,-3.29) .. controls (6.95,-1.4) and (3.31,-0.3) .. (0,0) .. controls (3.31,0.3) and (6.95,1.4) .. (10.93,3.29)   ;
\draw [shift={(183.5,69.19)}, rotate = 0] [color={rgb, 255:red, 0; green, 0; blue, 0 }  ][line width=0.75]    (10.93,-3.29) .. controls (6.95,-1.4) and (3.31,-0.3) .. (0,0) .. controls (3.31,0.3) and (6.95,1.4) .. (10.93,3.29)   ;
\draw  [draw opacity=0][fill={rgb, 255:red, 25; green, 163; blue, 226 }  ,fill opacity=1 ] (479.75,64.69) -- (485,73) -- (474.5,73) -- cycle ;
\draw  [draw opacity=0][fill={rgb, 255:red, 25; green, 163; blue, 226 }  ,fill opacity=1 ] (423.25,64.69) -- (428.5,73) -- (418,73) -- cycle ;
\draw  [draw opacity=0][fill={rgb, 255:red, 25; green, 163; blue, 226 }  ,fill opacity=1 ] (439.25,64.69) -- (444.5,73) -- (434,73) -- cycle ;
\draw  [draw opacity=0][fill={rgb, 255:red, 240; green, 101; blue, 37 }  ,fill opacity=1 ] (116.5,65.69) -- (123,65.69) -- (123,72.19) -- (116.5,72.19) -- cycle ;
\draw  [draw opacity=0][fill={rgb, 255:red, 240; green, 101; blue, 37 }  ,fill opacity=1 ] (137.5,66.09) -- (144,66.09) -- (144,72.59) -- (137.5,72.59) -- cycle ;
\draw  [draw opacity=0][fill={rgb, 255:red, 240; green, 101; blue, 37 }  ,fill opacity=1 ] (168,66.19) -- (174.5,66.19) -- (174.5,72.69) -- (168,72.69) -- cycle ;
\draw  [draw opacity=0][fill={rgb, 255:red, 240; green, 101; blue, 37 }  ,fill opacity=1 ] (148.5,66.19) -- (155,66.19) -- (155,72.69) -- (148.5,72.69) -- cycle ;
\draw  [draw opacity=0][fill={rgb, 255:red, 25; green, 163; blue, 226 }  ,fill opacity=1 ] (467.25,64.69) -- (472.5,73) -- (462,73) -- cycle ;
\draw [color={rgb, 255:red, 240; green, 101; blue, 37 }  ,draw opacity=1 ][line width=2.25]    (180.5,69.19) -- (116.5,69.19) ;
\draw    (421.5,49.69) -- (481.5,49.69) ;
\draw [shift={(483.5,49.69)}, rotate = 180] [color={rgb, 255:red, 0; green, 0; blue, 0 }  ][line width=0.75]    (10.93,-3.29) .. controls (6.95,-1.4) and (3.31,-0.3) .. (0,0) .. controls (3.31,0.3) and (6.95,1.4) .. (10.93,3.29)   ;
\draw [shift={(419.5,49.69)}, rotate = 0] [color={rgb, 255:red, 0; green, 0; blue, 0 }  ][line width=0.75]    (10.93,-3.29) .. controls (6.95,-1.4) and (3.31,-0.3) .. (0,0) .. controls (3.31,0.3) and (6.95,1.4) .. (10.93,3.29)   ;

\draw (143,28) node [anchor=north west][inner sep=0.75pt]   [align=left] {$\displaystyle D$};
\draw (288,56) node [anchor=north west][inner sep=0.75pt]   [align=left] {$\displaystyle r$};
\draw (124,93) node [anchor=north west][inner sep=0.75pt]   [align=left] {Class A};
\draw (426,93) node [anchor=north west][inner sep=0.75pt]   [align=left] {Class B};
\draw (443,28) node [anchor=north west][inner sep=0.75pt]   [align=left] {$\displaystyle D$};

\end{tikzpicture}
\caption{A simple example where we compute the optimal value of the abstention threshold exactly. Classes A and B are both distributed respectively on segments of length $D$, embedded collinear and at distance $r$ in $\R^2$.}
\label{figure: optimal threshold toy example}
\end{figure}

\noindent{\it Formal setting}: We set the feature and adversary dimensions as $n_2=2,n_3=1$. Examples of class A are all located on the segment $S_A=[(0,0),(D,0)]$, similarly instances of class B are located on $S_B=[(D+r,0),(2D+r,0)]$ (where $[\a,\b]:=\{\alpha\a+(1-\alpha)\b\mid \alpha\in[0,1]\}$). The data distribution returns an even number of samples, $2m$, with $m>0$ points each drawn uniformly from $S_A$ and $S_B$.
For this setting, we show that the optimal value of the threshold is a function of both the geometry ($D,r$) and the sampling rate ($m$). Proof of lemmas needed to prove the following result appear in Appendix \ref{appendix: toy example}.
\begin{theorem}\label{theorem: toy example optimal threshold}
Let $\tau^*:=\argmin_{\tau\in\R^+}l(\tau)$. For the setting considered above, if we further assume $D=o(r)$ and $m=\omega\left(\log\left(\frac{2\pi c r}{D}\right)\right)$, then there is a unique value of $\tau^*$ in $[0,D/2)$. Further, 
\begin{align*}
    \tau^*=\begin{cases}
    \Theta\left(\frac{D\log \left((\pi c r m)/D\right)}{m}\right),
    &
    \text{if }\frac{1}{m}<\frac{\pi c r}{D};\\
    0,
    &
    \text{if }\frac{\pi c r}{D}\le\frac{1}{m}.\\
    \end{cases}
\end{align*}
\end{theorem}

\begin{proof}
We compute the robust error $\cE_\adv(\tau)$ and abstention rate $\cD_\nat(\tau)$ as functions of $\tau$. Even with $D=o(r)$, the exact computation of the robust error as a simple closed form is difficult without further assuming $\tau=o(r)$ as well. Fortunately, by Lemma \ref{lemma: toy example}, we only need to consider $\tau\le D$. For this case, indeed $\tau=o(r)$. We compute the abstention and robust error rates in Lemmas \ref{lemma: toy example abstain} and \ref{lemma: toy example accuracy}, respectively. This gives us, for $\tau\le D$,
\begin{align*}
    l(\tau)=&\frac{\tau}{\pi r}\left(1-\frac{m+3}{m+1}\cdot\Theta\left(\frac{D}{r}\right)\right) -\Theta\left(\left(\frac{\tau}{r}\right)^3\right)\\&+\frac{c}{m+1}\left[2\left(1-\frac{\tau}{D}\right)^{m+1}+(m-1)\ind_{\tau\le D/2}\left(1-\frac{2\tau}{D}\right)^{m+1}\right].
\end{align*}
For $\tau\le D/2$,
\begin{align*}
    l'(\tau)=&\frac{1}{\pi r}\left(1-\frac{m+3}{m+1}\cdot\Theta\left(\frac{D}{r}\right)\right) -\Theta\left(\frac{1}{r}\left(\frac{\tau}{r}\right)^2\right)\\&-\frac{2c}{D}\left[\left(1-\frac{\tau}{D}\right)^{m}+(m-1)\left(1-\frac{2\tau}{D}\right)^{m}\right].
\end{align*}
We need to consider two cases.\\
{\it Case 1.} $\frac{\pi c r}{D}\le\frac{1}{m}$.
In this case $l'(0)=\frac{1}{\pi r}-\frac{2cm}{D}\ge 0$. Since $l''(\tau)\ge 0$, so we must have the only minimum at $\tau=0$.\\ \\
{\it Case 2.} $\frac{1}{m}<\frac{\pi c r}{D}$. $l'(0)=\frac{1}{\pi r}-\frac{2cm}{D}<0$. Also $l'(D/2)=\frac{1}{\pi r}-\frac{2c}{D2^m}>0$ since $m>\log\left(\frac{2\pi c r}{D}\right)$. But $l''(\tau)\ge 0$, so we must have a unique local minimum in $(0,D/2)$, which is the global minimum.\\
Further, define $y$ as $\tau=\frac{D}{m}\log y$. Now if $y=2^{o(m)}$, we have $\frac{\tau}{D}=o(1)$, or
\[\left(1-\frac{\tau}{D}\right)^{m}=\exp\left({m\log\left(1-\frac{\tau}{D}\right)}\right)=y^{-1-o(1)}.\]
If $y>1$, for $y=\frac{2\pi c r m}{D}$,
\begin{align*}
    l'(\tau)&=\frac{1}{\pi r}-\frac{2c}{D}\left[\left(\frac{D}{2\pi c r m}\right)^{1+o(1)}+(m-1)\left(\frac{D}{2\pi c r m}\right)^{2+o(1)}\right]\\
    &>\frac{1}{\pi r}-\frac{2c}{D}\left[\left(\frac{D}{2\pi c r m}\right)^{1}+(m-1)\left(\frac{D}{2\pi c r m}\right)^{1}\right]\\
    &=\frac{1}{\pi r}-\frac{2c}{D}\left[\frac{D}{2\pi c r}\right]=0,
\end{align*}
and for $y=\left(\frac{2\pi c r (m-1)}{D}\right)^{1/4}$,
\begin{align*}
    l'(\tau)&=\frac{1}{\pi r}-\frac{2c}{D}\left[\left(\frac{D}{2\pi c r (m-1)}\right)^{\frac{1}{4}+o(1)}+(m-1)\left(\frac{D}{2\pi c r m}\right)^{\frac{1}{2}+o(1)}\right]\\
    &<\frac{1}{\pi r}-\frac{2c}{D}\left[\left(\frac{D}{2\pi c r(m-1)}\right)^{1}+(m-1)\left(\frac{D}{2\pi c r(m-1)}\right)^{1}\right]\\
    &=\frac{-1}{\pi  r(m-1)}<0.
\end{align*}
Together, we get that $\tau^*=\Theta\left(\frac{D\log \left((\pi c r m)/D\right)}{m}\right)$ in this case.
\end{proof}

\vspace{-0.2cm}
\section{Experiments on Contrastive Learning}\label{sec:expt}

\blue{Contrastive learning has received significant attention due to the recent popularity of self-supervised learning: many recent studies~\citep{wu2018unsupervised,oord2018representation,hjelm2018learning,zhuang2019local,henaff2019data,tian2019contrastive,bachman2019learning} present promising results of unsupervised representation learning against their supervised counterparts. Representative self-supervised contrastive learning includes MoCo(v2)~\citep{he2020momentum} and SimCLR~\citep{chen2020simple}. In ImageNet classification task, both methods almost match the accuracy of their supervised counterparts; in 7 detection/segmentation tasks on PASCAL VOC, COCO, and other data sets, MoCo~\citep{he2020momentum} can outperform its supervised pre-training counterpart sometimes by large margins. A more recent work of \cite{khosla2020supervised} proposed \emph{supervised contrastive learning}.}

Theorem \ref{theorem: positive results} sheds light on how to design algorithms for robust learning of feature embedding $F$. In order to preserve robustness against adversarial examples regarding a given test point $\x$, in the feature space the theorem suggests minimizing $\tau$---the closest distance between $F(\x)$ and any training example $F(\x_i)$ with the same label, and maximizing $r$---the closest distance between $F(\x)$ and any training example $F(\x_i)$ with a different label. \blue{This is conceptually consistent with the spirit of the nearest-neighbor algorithm. Indeed,  contrastive loss can be seen as nearest-neighbor loss (in the feature space) with the \emph{max} operator replaced by a \emph{softmax} operator for differentiable training}:
\begin{equation}
\label{equ: sn loss}
\min_F -\frac{1}{m}\sum_{i\in [m]} \log\left(\frac{\sum_{j\in[m],j\not=i,y_i=y_j}e^{-\frac{\|F(\x_i)-F(\x_j)\|^2}{T}}}{\sum_{k\in[m],k\not= i} e^{-\frac{\|F(\x_i)-F(\x_k)\|^2}{T}}}\right),
\end{equation}
where $T>0$ is the temperature parameter.
Loss (\ref{equ: sn loss}) is also known as the soft-nearest-neighbor loss in the context of supervised learning~\citep{frosst2019analyzing}, or the InfoNCE loss in the setting of self-supervised learning~\citep{he2020momentum}.

We will now describe an implementation of the attack and empirically measure the performance of our algorithm in the context of supervised and self-supervised contrastive learning\footnote{Code used in the experiments may be found at the following github link: \url{https://github.com/dravyanshsharma/adversarial-contrastive}}.

\comment{
: in the supervised learning setting, we use the true labels $y_i$'s and let $n=m$ for $m$ training examples sampled from $\cD_\cX$. In the setting of self-supervised learning, we sample two separate data augmentation operators from the same family of augmentations and apply them to each training data example to obtain two correlated views. We use the \emph{artificial} labels $y_i$'s: we label the two views of the same data example by the same class, so $|\cY|=m$ and $n=2m$.
}

\comment{
\medskip
\noindent{\textit{Intuition behind the optimization.}} Problem \eqref{equ: sn loss} captures the trade-off between accuracy and robustness: the numerator encourages the intra-class distance to be minimized for natural accuracy, while the denominator encourages the inter-class distance to be maximized for adversarial robustness. The (soft) $k$-nearest neighborhood of $F(\x_i)$ is encouraged to be the features from the same class. 
}

\subsection{Visualization of Representations of Contrastive Learning}

Figure \ref{figure: tsne} shows the two-dimensional t-SNE visualization of 10,000 features by minimizing loss (\ref{equ: sn loss}) on the CIFAR10 test data set. It shows that $\tau_x\ll r_x$ for most of data, where we define $\tau_x:=\min_{i:y=y_i} \dist(F(\x),F(\x_i))$, $r_x:=\min_{i:y\not=y_i} \dist(F(\x),F(\x_i))$, and $\{\x_i\}_{i=1}^m$ is a set of training example with labels $y_i$.

\begin{figure}[t]
\centering
\subfigure{
\includegraphics[scale=0.45]{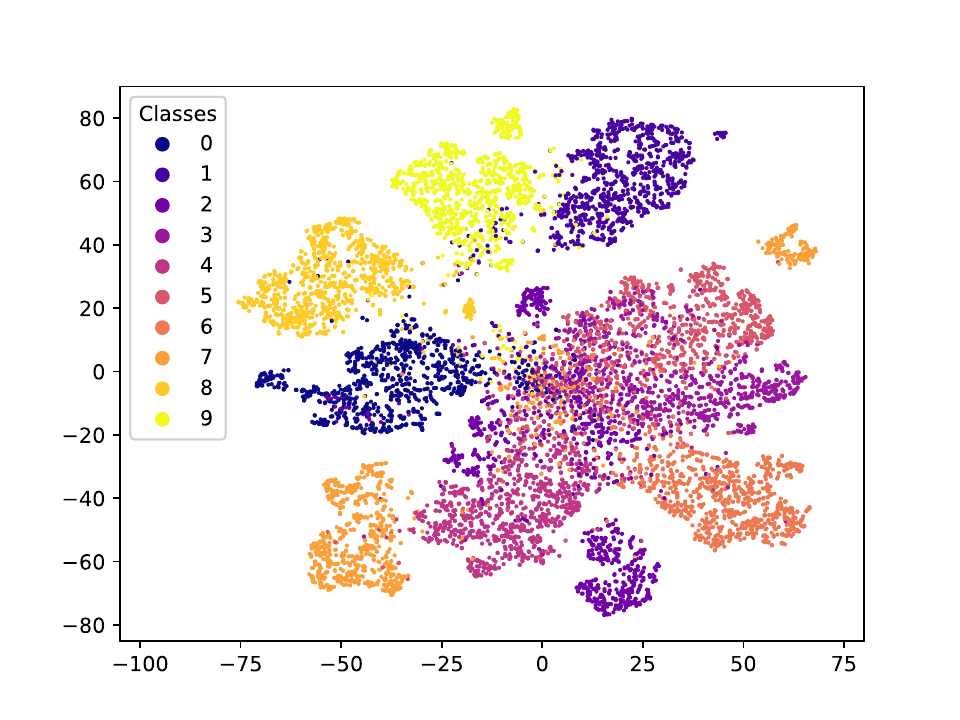}}
\subfigure{
\includegraphics[scale=0.45]{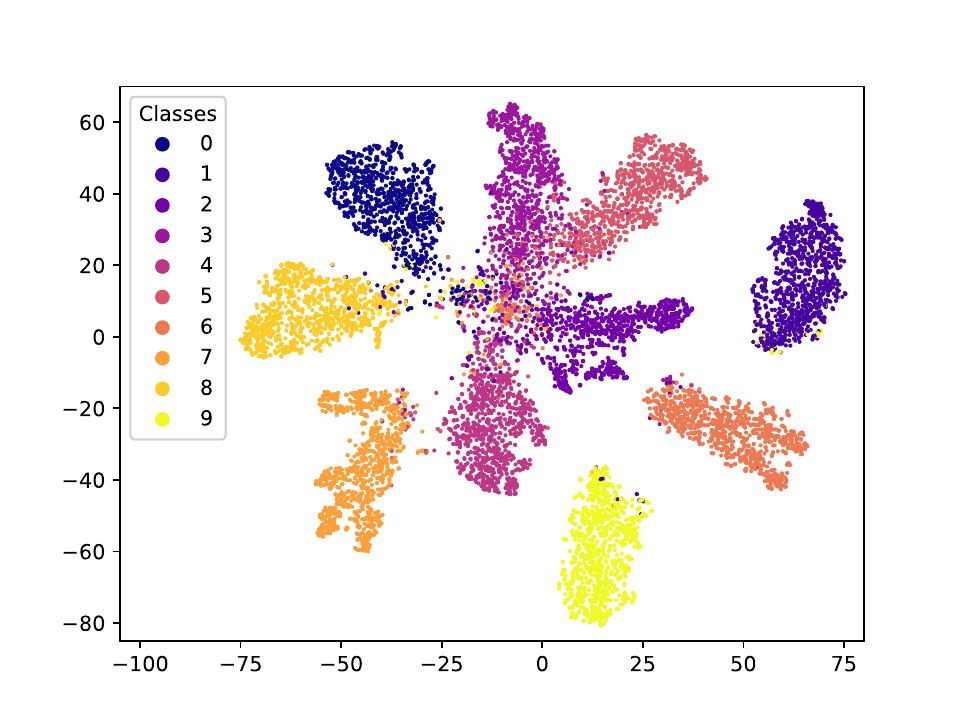}}
\caption{Two-dimensional t-SNE visualization of 512-dimensional embedding by contrastive learning on the CIFAR10 test data set. \textbf{Left Figure:} Self-supervised contrastive learning. \textbf{Right Figure:} Supervised contrastive learning.}
\label{figure: tsne}
\end{figure}

To have a closer look at $\tau_x$ vs. $r_x$, we plot the frequency of $\tau_x/r_x$ in Figure \ref{figure: tau/r}. For self-supervised contrastive learning, there is 84.5\% data which has $\tau_x/r_x$ smaller than $1.0$, while for supervised setting, there is 94.3\% data which has $\tau_x/r_x$ smaller than $1.0$.

\begin{figure}[t]
\centering
\subfigure{
\includegraphics[scale=0.45]{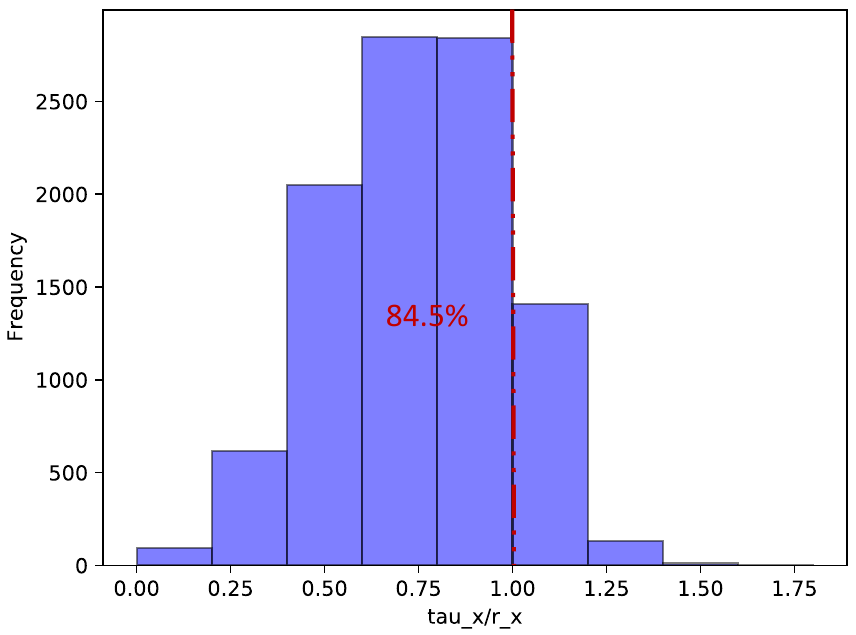}}
\subfigure{
\includegraphics[scale=0.45]{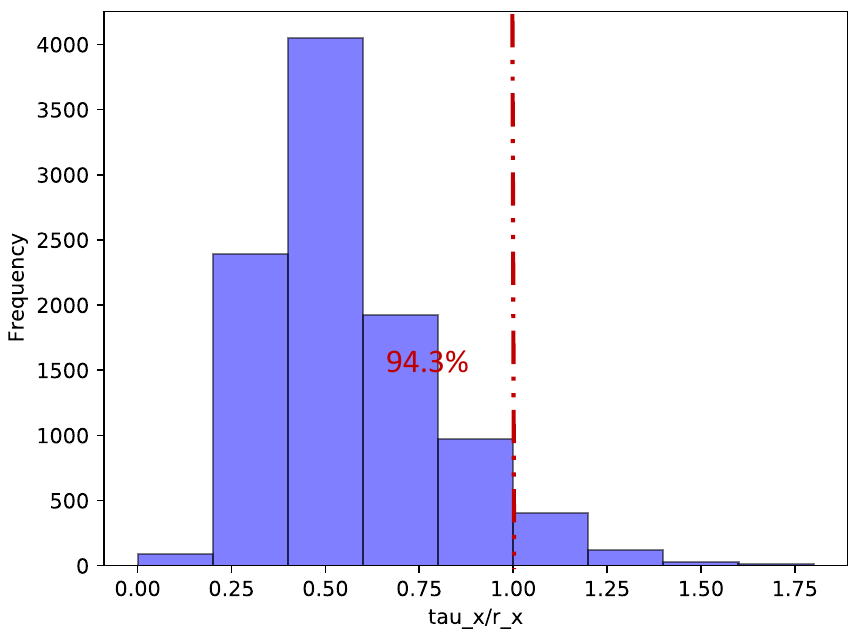}}
\caption{Frequency of $\tau_x/r_x$ by contrastive learning on the CIFAR10 data set, where $\tau_x$ represents the closest distance between the test embedding and any training embedding of the same label, and $r_x$ stands for the closest distance between the test embedding and any training embedding of different labels. \textbf{Left Figure:} Self-supervised contrastive learning. \textbf{Right Figure:} Supervised contrastive learning.}
\label{figure: tau/r}
\end{figure}

\subsection{Certified Adversarial Robustness against Exact Computation of Attacks}

\label{section: certified adversarial robustness}

We verify the robustness of Algorithm \ref{algorithm: robust separated inference-time classifier} when the representations are learned by contrastive learning. Given a embedding function $F$ and a classifier $f$ which outputs either a predicted class or abstains from predicting, recall that we define the natural and robust errors, respectively, as $\cE_\nat(f):=\mathbb{E}_{(\x,\y)\sim\cD}\1\{f(F(\x))\ne\y \text{ and $f(F(\x))$ does not abstain}\}$, and $\cE_\adv(f):=\mathbb{E}_{(\x,\y)\sim\cD,S\sim\cS}\1\{\exists \e\in S+F(\x)\subseteq \R^{n_2} \text{ s.t. } f(\e)\ne\y \text{ and $f(\e)$ does not abstain}\}$, where $S\sim\cS$ is a random adversarial subspace of $\R^{n_2}$ with dimension $n_3$. $\cD_\nat(f):=\mathbb{E}_{(\x,\y)\sim\cD}\1\{f(F(\x))\text{ abstains}\}$ is the abstention rate on the natural examples. Note that the robust error is always at least as large as the natural error.

\begin{table*}[ht]
	\centering
	\caption{Natural error $\cE_\nat$ and robust error $\cE_\adv$ on the CIFAR-10 data set \citep{szegedy2015going} when $n_3=1$ and the 512-dimensional representations are learned by contrastive learning, where $\cD_\nat$ represents the fraction of each algorithm's output of ``don't know'' on the natural data. We report values for $\sigma\approx\tau$ as they tend to give a good abstention-error trade-off w.r.t.  $\sigma$. Bold values correspond to parameter settings that minimize $\cE_\adv+\cD_\nat$ over the grid.}
	\label{table: robustness n3 1}
	\resizebox{1.0\textwidth}{!}{%
\begin{tabular}{@{}cr||cc|ccc|ccc@{}}
\toprule
& Contrastive & \multicolumn{2}{c}{Linear Protocol} & \multicolumn{3}{c}{Ours ($\tau=3.0$)}& \multicolumn{3}{c}{Ours ($\tau=2.0$)} \\
& & $\cE_{\nat}$ & $\cE_\adv$ & $\cE_{\nat}$ & $\cE_\adv$ & $\cD_\nat$ & $\cE_{\nat}$ & $\cE_\adv$ & $\cD_\nat$\\
\midrule
\midrule
\multirow{2}{1.8cm}{$(\sigma=0)$} &
Self-supervised & 8.9\% & 100.0\% & 15.4\% & 40.7\% & 2.2\% &  14.3\% & 26.2\%& 28.7\%\\
& Supervised & 5.6\% & 100.0\% & 5.7\% & 60.5\% & 0.0\% & 5.7\% & 33.4\%& 0.0\%\\
\midrule
\multirow{2}{1.8cm}{$(\sigma=0.9\tau)$} &
Self-supervised & 8.9\% & 100.0\% & \bf 7.2\% & \bf 9.4\% & \bf 12.9\% &  10.0\% & 17.7\%& 29.9\%\\
& Supervised & 5.6\% & 100.0\% & 6.2\% & 18.9\% & 0.0\% & 5.6\% & 22.0\%& 0.1\%\\
\midrule
\multirow{2}{1.8cm}{$(\sigma=\tau)$} &
Self-supervised & 8.9\% & 100.0\% & 1.1\% & 1.2\% & 33.4\% &  2.1\% & 3.1\%& 49.9\%\\
& Supervised & 5.6\% & 100.0\% & 1.9\% & 2.8\% & 10.6\% & \bf 4.1\% & \bf 4.8\%& \bf 3.3\%\\
\bottomrule
\end{tabular}
}
\vspace{-0.3cm}
\end{table*}

\medskip
\noindent{\textit{Self-supervised contrastive learning setup.}} Our experimental setup follows that of SimCLR \citep{chen2020simple}. We use the ResNet-18 architecture~\citep{he2016deep} for representation learning with a two-layer projection head of width 128. The dimension of the representations is 512. We set batch size 512, temperature $T=0.5$, and initial learning rate 0.5 which is followed by cosine learning rate decay. We sequentially apply four simple augmentations: random cropping followed by resizing back to the original size, random flipping, random color distortions, and randomly converting image to grayscale with a probability of 0.2. In the linear evaluation protocol, we set batch size 512 and learning rate 1.0 to learn a linear classifier in the feature space by empirical risk minimization. All experiments are run on two GeForce RTX 2080 GPUs.

\medskip
\noindent{\textit{Supervised contrastive learning setup.}} Our experimental setup follows that of \cite{khosla2020supervised}. We use the ResNet-18 architecture for representation learning with a two-layer projection head of width 128. The dimension of the representations is 512. We set batch size 512, temperature $T=0.1$, and initial learning rate 0.5 which is followed by cosine learning rate decay. We sequentially apply four simple augmentations: random cropping followed by resize back to the original size, random flipping, random color distortions, and randomly converting image to grayscale with a probability of 0.2. In the linear evaluation protocol, we set batch size 512 and learning rate 5.0 to learn a linear classifier in the feature space by empirical risk minimization.

\medskip
\noindent{\textit{Algorithm for exact implementation of the attack.}} In both self-supervised and supervised setups, we compare the robustness of the linear protocol with that of our defense protocol in Algorithm \ref{algorithm: robust separated inference-time classifier} under exact computation of adversarial examples using a convex optimization program in $n_3$ dimensions and $m$ constraints. Algorithm \ref{algorithm: adversarial attack opt} provides an efficient implementation of the attack. 

\begin{algorithm}[!h]
\caption{Exact computation of attacks under threat model \ref{sec: threat model} against Algorithm \ref{algorithm: robust separated inference-time classifier}}
\label{algorithm: adversarial attack opt}
\begin{algorithmic}[1]
\STATE {\bfseries Input:} A randomly-sampled adversarial subspace $S$ of dimension $n_3$, a test example $F(\x)$ and its label $y$, a set of training examples $F(\x_i)$ and their labels $y_i$, $i\in[m]$, a threshold parameter $\tau$.
\STATE {\bfseries Output:} A misclassified adversarial feature $F(\x)+\v$, $\v\in S$ if it exists; otherwise, output ``no adversarial example''.
\STATE{$F_{\mathrm{center}}(\x_i)\leftarrow F(\x_i)-F(\x)$ for $i\in[m]$.}
\FOR{$i=1,...,m$}
\IF{$y_i\ne y$ }
\STATE{$\u_i=\argmin_{\u\in S}d(\u,F_{\mathrm{center}}(\x_i))$;\hspace*{2.28cm}(candidate adversarial perturbation)}
\STATE{$C\leftarrow\{\x_j\mid y_j=y\}$;}
\IF{$\exists \w \in C \mid \dist(\u_i,F_{\mathrm{center}}(\w))<\dist(\u_i,F_{\mathrm{center}}(\x_i))$}
\STATE{$H_j\leftarrow \{\z\mid \dist(F_{\mathrm{center}}(\x_i),\z)\le \dist(\w_j,\z), \w_j\in C\}$;} 
\STATE{$H\leftarrow \cap_iH_i$;}
\STATE{$A\leftarrow H\cap S$;}

\IF{$A=\{\}$}
\STATE{\bfseries continue};
\ENDIF

\STATE{$\z_i=\argmin_{\z\in A}\dist(\z,F_{\mathrm{center}}(\x_i))$;\hspace*{1.75cm}(candidate adversarial perturbation)}\label{line:convex-opt}
\ELSE
\STATE{$\z_i\leftarrow\u_i$;}
\ENDIF
\IF{$\dist(\z_i,F_{\mathrm{center}}(\x_i))<\tau$}
\STATE{\textbf{return} $F(\x)+\z_i$.}
\ENDIF
\ENDIF
\ENDFOR
\STATE{\textbf{return} ``no adversarial example''.}
\end{algorithmic}
\end{algorithm}

{\it Overview of Algorithm \ref{algorithm: adversarial attack opt}}. If the point $\u_i$ closest to the training point $\x_i$ of different label than test point $\x$ in the adversarial subspace $S$ (slight abuse of notation to refer to $\x+S$ as $S$) is closer to $\x_i$ than any training point $\w_j$ with the same label as $\x$ and within the threshold $\tau$ of $\x_i$, it will be misclassified as $\x_i$ (or potentially another point of an incorrect label). If however $\u_i$ is closer to some $\w_j$, we look at the points closer to $\x_i$ than all $\w_j$ in the subspace $S$, and consider the closest point $\z_i$ to $\x_i$ (if it is within threshold $\tau$) which should be misclassified. This can be computed using a convex optimization program (Line \ref{line:convex-opt} of Algorithm \ref{algorithm: adversarial attack opt}) in $n_3$ dimensions. We claim it is sufficient to look at these two points for each training example $\x_i$. 

{\it Proof of correctness}. To argue correctness of Algorithm \ref{algorithm: adversarial attack opt}, suppose an adversary wins by perturbing to some point $\v$. Then $\v$ must be closer to some point $\x_i$ than all $\w_j\in C$ (the set of training points with same label as $\x$) and within $\tau$ of $\x_i$. If $\u_i$ is closer to $\x_i$ than all $\w_j\in C$ then, it must be at least as close as $\v$ (since $\v$ is in the adversarial subspace $S$) and therefore within $\tau$ of $\x_i$.

Otherwise there is some $\w_j$ closer to $\u_i$ than $\x_i$. Let $H$ be the convex polytope of points closer to $\x_i$ than $\w_j$'s in $C$. Consider the intersection $A$ of $H$ with $S$. All points in $A$ are misclassified by our algorithm, if within the threshold $\tau$. $\v$ must lie within $A$ since it is closer to $\x_i$. $\u_i$ must lie outside of $A$ in this case. If $\v$ is within $\tau$ of $\x_i$, so is $\u_i$ and therefore also the line joining the two. If this line intersects $A$ at point $\v$, then $\v$ is a valid adversarial point and so is point closest to $\x_i$ in $A$. This proves completeness of the algorithm, soundness is more straightforward to verify.

\comment{
{\it Overview of Algorithm \ref{algorithm: adversarial attack opt}}: If the point $\u_i$ closest to the training point $\x_i$ of different label than test point $\x$ in the adversarial subspace $\cS$ (slight abuse of notation to refer to $\x+\cS$ as $\cS$) is closer to $\x_i$ than any training point $\w_j$ with the same label as $\x$ and within the threshold $\tau$ of $\x_i$, it will be misclassified as $\x_i$ (or potentially another point of an incorrect label). If however $\u_i$ is closer to some $\w_j$, we look at the points closer to $\x_i$ than all $\w_j$ in the subspace $\cS$, and consider the closest point $\z_i$ to $\x_i$ (if it is within threshold $\tau$) which should be misclassified. This can be computed using a convex optimization program (Line \ref{line:convex-opt} of Algorithm \ref{algorithm: adversarial attack opt}) in $n_3$ dimensions. We claim it is sufficient to look at these two points for each training example $\x_i$. See Appendix \ref{app:proof-attack-algo} for a proof of correctness.
}

\medskip
\noindent\textit{Experimental results.} We summarize our results in Table \ref{table: robustness n3 1}. Comparing with a linear protocol, our algorithms have much lower robust error. Note that even if abstention is added based on distance from the linear boundary, sufficiently large perturbations will ensure the adversary can always succeed. 
For an approximate adversary which can be efficiently implemented for large $n_3$, see Appendix \ref{app:attack-large-n3}.

\begin{figure}[t]
\centering
\includegraphics[scale=0.5]{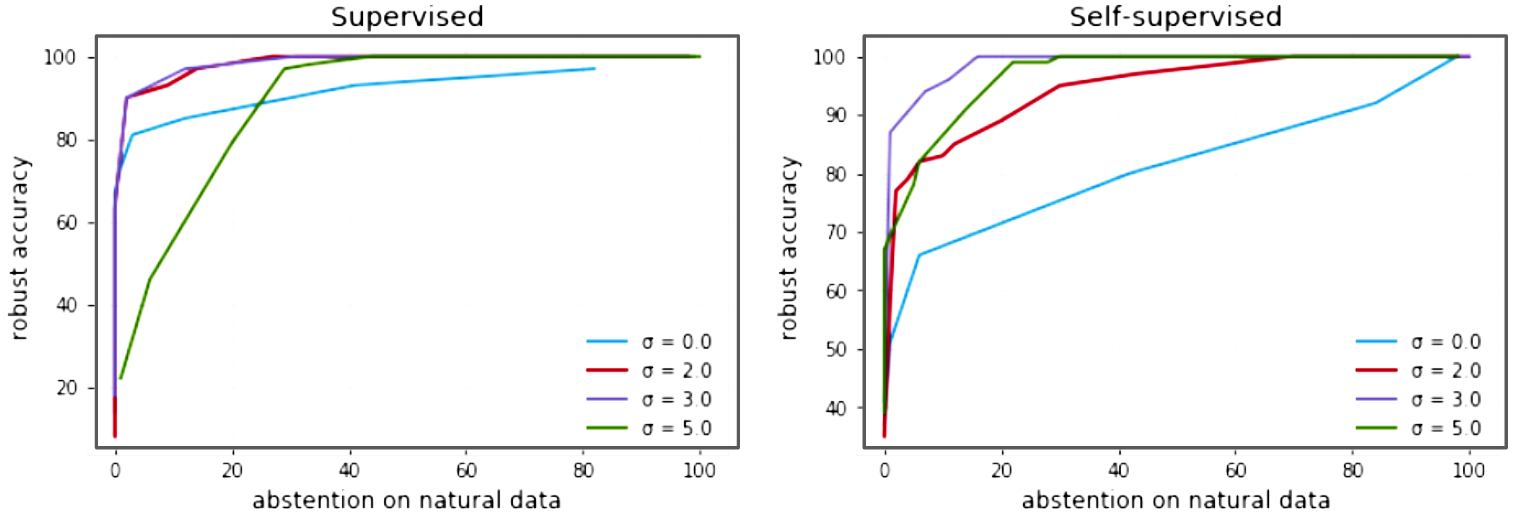}
\caption{Adversarial accuracy (that is, rate of adversary failure) vs. abstention rate as threshold $\tau$ varies for $n_3=1$ and different outlier removal thresholds $\sigma$. Each colored line corresponds to a fixed $\sigma$, as $\tau$ is varied from 0 (always abstain) to infinity (vanilla nearest-neighbor).}
\label{figure: tau1}
\end{figure}

\subsection{Robustness-abstention Trade-off}
The threshold parameter $\tau$ captures the trade-off between the robust accuracy $\cA_{\adv}:=1-\cE_{\adv}$ and the abstention rate $\cD_{\nat}$ on the natural data. We report both metrics for different values of $\tau$ for supervised and self-supervised contrastive learning. The supervised setting enjoys higher adversarial accuracy and a smaller abstention rate for fixed $\tau$'s due to the use of extra label information. We plot $\cA_{\adv}$ against $\cD_{\nat}$ for Algorithm \ref{algorithm: robust separated inference-time classifier} as hyperparameters vary. For small $\tau$, both accuracy and abstention rate approach 1.0. As the threshold increases, the abstention rate decreases rapidly and our algorithm enjoys good accuracy even with small abstention rates. For $\tau\rightarrow\infty$ (that is the nearest neighbor search), the abstention rate on the natural data $\cD_\nat$ is $0\%$ but the robust accuracy is also roughly $0\%$. 
Increasing $\sigma$ (for small $\sigma$) gives us higher robust accuracy for the same abstention rate. Too large $\sigma$ may also lead to degraded performance (Figure \ref{figure: tau1}). 

\blue{
\section{Discussion and Conclusion}
\label{sec:discussion}

We propose a model to study robustness of non-Lipschitz networks, against an adversary whose perturbations  modify the features in a random low-dimensional subspace.
Our first result is that in our model if the learner does not use any abstention, then the adversary
will succeed for any data distribution. To complement our lower bound, we present a threshold-equipped nearest-neighbor classifier that simultaneously achieves low robust error as well as low abstention rate on natural data. Our robust error guarantee is independent of the distribution, and is small as long as the label classes are well-separated in the feature space. Our bounds for  abstention rate scale with the covering number of the distribution, and hold for sufficiently large training set size $m$. Our positive results indicate a trade-off between the robust error and abstention rate. 
We further show how one may tune the threshold to minimize a combination of robust error and abstention rate using techniques from data-driven algorithm design. We also validate our positive results empirically for contrastive learning based deep networks.

Adversarial robustness is an important challenge for the practical deployment of deep networks. We believe we should analyze different types of adversaries beyond classic ones ($\ell_\infty$, $\ell_2$, $\ell_1$ bounded-norm perturbations) which have largely been the focus in our community.  We view our contribution as defining and analyzing a new and interesting type of adversary designed to help in  studying the robustness of non-Lipschitz networks.  It is an interesting open question to provide new families of adversaries as well as defenses for them, since bounded-norm models are limited in their ability to capture all possible realistic attacks.}

\section*{Acknowledgments}
This work was supported in part by the National Science Foundation under grants CCF-1815011, CCF-1535967, IIS-1901403, CCF-1910321, SES-1919453, the Defense Advanced Research Projects Agency under cooperative agreement HR00112020003, the NSERC Discovery Grant RGPIN-2022-03215, DGECR-2022-00357, an AWS Machine Learning Research Award, a Microsoft Research Faculty Fellowship, and a Bloomberg Research Grant. The views expressed in this work do not necessarily reflect the position or the policy of the Government and no official endorsement should be inferred. Approved for public release; distribution is unlimited.

\bibliographystyle{plainnat}
\bibliography{reference.bib}

\newpage
\appendix

\section{Technical Lemmas Needed for Results in Section \ref{sec: positive results}}
\label{sec:upper bound beta function}

\blue{The following lemma gives a bound on the fraction of the surface of the sphere at some fixed small distance from the subspace in Theorem \ref{theorem: improved bound}. The bound involves a geometric calculation of a surface element of a sphere in $\R^n$.}

\begin{lemma}\label{lem:hsphere} The fraction of the surface of the unit $(n-1)$-sphere at a distance at most small $\varepsilon=o(1)$ from a fixed $(n-k)$-hyperplane through its center is at most $\frac{2\varepsilon^k}{k}\cdot\frac{A(k-1)A(n-k-1)}{A(n-1)}$, where $A(m)$ is the surface-area of the unit $m$-sphere embedded in $m+1$ dimensions.

\end{lemma}
\begin{proof}
Let the fixed hyperplane be $x_{1}=x_{2}=\dots=x_{k}=0$. We change the coordinates to a product of spherical coordinates ($\rho$ is the distance from the hyperplane, $r$ is the orthogonal component of the radius vector).

\[x_j=\begin{cases} 
\rho S_{j-1}\cos\phi_j, & \mbox{if } j < k; \\
\rho S_{j-1}, & \mbox{if } j = k; \\
r T_{j-k-1}\cos\alpha_{j-k}, & \mbox{if }  k < j < n; \\
r T_{j-k-1}, & \mbox{if }  j = n.
\end{cases}
\]
where $S_l=\prod_{i=1}^l\sin\phi_i$ and $T_l=\prod_{i=1}^l\sin\alpha_i$. The desired surface area is easier to compute in the new coordinate system.

The new coordinates are $(y_1,\dots,y_n)=(\rho,\phi_1,\phi_2,\dots,\phi_{k-1},r,\alpha_1,\dots,\alpha_{n-k-1})$. Let $z=\sqrt{r^2+\rho^2}=\sqrt{\sum_{i=1}^nx_i^2}$ denote the usual radial spherical coordinate. Volume element in this new coordinate system is given by
$$dV=|\det(J)|\,d\rho\,  d\phi_1\dots d\phi_{k-1} dr\, d\alpha_1\dots d\alpha_{n-k-1},$$
where $J$ is the Jacobian matrix, $J_{ij}=\frac{\partial x_i}{\partial y_j}$.
We can write
\[
J=\begin{bmatrix}
A&0\\
0&B
\end{bmatrix},
\]
where $A_{ij}=\frac{\partial x_i}{\partial y_j}$ for $1\le i,j\le k$ and $B_{ij}=\frac{\partial x_{i+k}}{\partial y_{j+k}}$ for $1\le i,j\le n-k$.
By Leibniz formula for determinants, it is easy to see
\begin{align*}\det(J)&=\det(A)\cdot\det(B)\\
&=\rho^{k-1}\left(\prod_{i=1}^{k-2}\sin^{k-i-1}\phi_i\right)\cdot r^{n-k-1}\left(\prod_{i=1}^{n-k-2}\sin^{n-k-i-1}\alpha_i\right)\\
&=\rho^{k-1}r^{n-k-1}\left(\prod_{i=1}^{k-2}\sin^{k-i-1}\phi_i\right)\left(\prod_{i=1}^{n-k-2}\sin^{n-k-i-1}\alpha_i\right).
\end{align*}
Now the surface element is given by $$dS=\frac{1}{z^{n-1}}\frac{dV}{dz}=\frac{1}{z^{n-1}}\left(\frac{dV}{dr}\frac{\partial r}{\partial z}+\frac{dV}{d\rho}\frac{\partial \rho}{\partial z}\right)=\frac{1}{rz^{n-2}}\frac{dV}{dr}+\frac{1}{\rho z^{n-2}}\frac{dV}{d\rho}.$$
Plugging in our computation for $dV$,
$$dS=\left(\frac{\rho^{k-1}r^{n-k-2}}{z^{n-2}}\,d\rho+\frac{\rho^{k-2}r^{n-k-1}}{z^{n-2}}\,dr\right) \left(\prod_{i=1}^{k-2}\sin^{k-i-1}\phi_id\phi_i\right)\left(\prod_{i=1}^{n-k-2}\sin^{n-k-i-1}\alpha_id\alpha_1\right).$$
We care about $z=1$ and $\rho\le\varepsilon$ (or $r\ge \sqrt{1-\varepsilon^2}$). Notice
$$\int_{\sqrt{1-\varepsilon^2}}^1\frac{\rho^{k-2}r^{n-k-1}}{z^{n-2}}\,dr= \int_{\varepsilon}^0\rho^{k-2}r^{n-k-1}\frac{-\rho d\rho}{r}=\int_0^{\varepsilon}\rho^{k-1}r^{n-k-2}d\rho.$$

Thus, using the surface element in the new coordinates and integrating, we get
\[\text{Area of $\varepsilon$-close points}
=A(k-1)A(n-k-1)\cdot 2\int_0^{\varepsilon}\rho^{k-1}r^{n-k-2}d\rho\le A(k-1)A(n-k-1)\cdot \frac{2\varepsilon^k}{k}\]
which gives the desired fraction.
\end{proof}

\noindent\blue{The following lemma establishes a useful convexity property for the adversarial linear subspaces.}

\begin{lemma}\label{lem: convex adversary}
Let $\x,\x'\in \R^{n_2}, \tau\in \R^+$ and $\cS(x',\tau)$ denote the subset of linear subspaces of dimension $n_3$ such that for any $S\in\cS(x',\tau)$ there exists $v\in S$ with $x+v\in \cB(x',\tau)$. The set $\cS(x',\tau)$ is convex.
\end{lemma}
\begin{proof}
Let $S,S'\in \cS(x',\tau)$. Then we have $v\in S, v'\in S'$ such that $x+v,x+v'\in \cB(x',\tau)$. Let $S^*=\alpha S+(1-\alpha)S'$, $\alpha\in[0,1]$. Pick $v^*=\alpha v+(1-\alpha)v' \in S^*$. $x+v^*$ must lie in $\cB(x',\tau)$ by convexity of $\cB(x',\tau)$.
\end{proof}

\section{Error Upper Bound with Outlier Removal}\label{app:outlier}
\blue{
Our results will be good for distributions for which the induced distribution $\cD_{\sigma}$ after the preprocessing step of Algorithm \ref{algorithm: robust separated inference-time classifier} satisfies the following property with small $N=\sum_y|\cB^y|$.

\begin{definition} 
\label{definition: separably covering}
A distribution $\cD$ is $\sigma$-separably $\{\cB^y\}$-coverable if all points in the support of the marginal distribution $\cD_{F(\cX)\mid y}$ over $\R^{n_2}$ can be covered by balls in the set $\cB^y=\{\bbB_1^y$, \dots, $\bbB_{N_y}^y\}$, of radius $\tau/2$ such that $$\min_{\substack{F(\x)\in \bbB_i^y,F(\x')\in \bbB_j^{y'},\\y\ne y'}} \dist(F(\x),F(\x'))\ge \sigma.$$
\end{definition}

\noindent In addition, we will assume that a test point $(\x,y)$ from the natural distribution $\cD$ has the property that $\x$ is covered by some ball in $\cB^y$ with high probability. 

\begin{theorem}\label{theorem: robust accuracy with separation}
Suppose the  distribution $\cD_{\sigma}$ induced by the preprocessing step of Algorithm \ref{algorithm: robust separated inference-time classifier} is $\sigma$-separably $\{\cB^y\}$-coverable with finite $N=\sum_y|\cB^y|$. Let $\Pr_{\x,y\sim \cD}[\x\in\cup_{\bbB_i\in\cB^y}\bbB_i]\ge 1-\gamma$.  If $\tau=o(\sigma)$, the robust error of Algorithm \ref{algorithm: robust separated inference-time classifier} on any test point $\x\sim\cD_{F(\cX)}$ is at most $${\cO}\left(
N\Bigg(\frac{c\tau}{(\sigma+\tau/2)\sqrt{1-\frac{n_3}{n_2}}}\Bigg)^{n_2-n_3}+Nc_0^{n_2-n_3} +\gamma\right),$$
where $c>0$ and $0<c_0<1$ are absolute constants.
\end{theorem}
\begin{proof}
Let $\x,y\sim \cD$. We will bound the probability the adversary succeeds for a  test point $\x$ covered  by $\cup_y\cB^y$, that is $\Pr[\text{adversary succeeds on }\x\mid \x\in\cup_{\bbB_i\in\cB^y}\bbB_i]$. Let $\x_i$ be a training point that survives the preprocessing step of Algorithm \ref{algorithm: robust separated inference-time classifier}, and belongs to a different class than $\x$. By the covering assumption, $\x_i\in\bbB^{y'}_j$ for some $y'\ne y$ and $\bbB^{y'}_j\in\cB^{y'}$. Let $\c$ denote the center of $\bbB^{y'}_j$. By the $\sigma$-separable property, we have $\dist(\x,\c)\ge \sigma+\tau/2$. Moreover, to succeed by perturbing close to any training point in $\bbB^{y'}_j$, the adversary must perturb to a point at distance at most $\tau+\tau/2=3\tau/2$ from $\c$ (by triangle inequality).

Using the same argument as in Theorem \ref{theorem: positive results}, the adversary succeeds in causing misclassification by perturbing $\x$ close to a point in $\bbB^{y'}_j$ with probability at most 
$$\Bigg(\frac{c\tau}{(\sigma+\tau/2)\sqrt{1-\frac{n_3}{n_2}}}\Bigg)^{n_2-n_3}+c_0^{n_2-n_3}$$
over the randomness of the adversarial subspace, for absolute constants $c>0$ and $0<c_0<1$. By a union bound, the adversary's success probability is at most $N$ times the above quantity, conditioned on $\x\in\cup_{\bbB_i\in\cB^y}\bbB_i$. Finally by assumption $\Pr_{\x,y\sim \cD}[\x\notin\cup_{\bbB_i\in\cB^y}\bbB_i]\le \gamma$, and using the law of total probability we get the desired upper bound.
\end{proof}
}

\section{New Lemmas and Results from Prior Work needed to prove Theorem \ref{thm:regret}}\label{app:dd}
\blue{We begin with an observation, which allows us to focus on small $\tau$. In particular we note that the nearest-neighbor distance for most points is $O(m^{-1/n_2})$, and therefore searching for threshold in the range $[0,\tau_{\max}]$ with $\tau_{\max}=O(m^{-1/n_2})$ is sufficient for almost no abstention. This can provide a useful guide in setting $\tau_{\max}$ in Theorem \ref{thm:regret}.} 
To simplify our results, we will treat $n_2,n_3$ as constants in the following.  
\begin{lemma}\label{lem:small-tau}
Let $\Phi$ be a distribution defined on a compact convex subset $C$ of $\R^n$ whose density function $\phi$ is continuous and strictly positive  on $C$ (that is $\phi(\x)>0$ for $\x\in C$), and has bounded partial derivatives throughout $C$. If $m$ samples $B=\{\beta_1,\dots,\beta_m\}$ are drawn from $\Phi$, for any $\beta_i$ the probability that the distance $d_i$ to its nearest neighbor in $B$ is not $O(m^{-1/n})$ is $o(1)$.
\end{lemma}
\begin{proof}
We use the asymptotic moments of nearest neighbor distance distribution due to \cite{evans2002asymptotic} together with a concentration inequality to complete the proof. Indeed, the asymptotic mean nearest neighbor distance is shown to be $O(m^{-1/n})$, and the variance is $O(m^{-2/n})$. By Chebyshev's inequality, the probability that $d_i$ is outside $\omega(1)$ standard deviations is $o(1)$.
\end{proof}

\noindent\blue{We will need the following lemma about Lipschitzness of $\cE_\adv(\tau)$. The argument can also be adapted to bounded density adversary (Corollary \ref{cor:bounded density}), and to show a bound on the breakpoints in $\hat{\cE}_\adv(\tau)$ (Corollary \ref{cor:estimated-loss}).}

\begin{lemma}\label{lem:acc-lipschitzness}
If $\tau\le \tau_{\max}=o\left(r\right)$, $\cE_\adv(\tau)$ is $O\left(m\tau_{\max}^{n_2-n_3-1}/r^{n_2-n_3}\right)$-Lipschitz.
\end{lemma}

\begin{proof}
Consider the probability that the adversary is able to succeed in misclassifying a test point $x$ as a fixed training point $x'$ (of different label) only when the threshold increases from $\tau$ to $\tau+d\tau$. Scale all distances by a factor of $\frac{1}{\dist(x,x')}=:\frac{1}{r'}$. WLOG, let $x$ be the origin and the adversarial subspace $S$ be given by $x_{n_3+1}=x_{n_3+2}=\dots=x_{n_2}=0$, and $x'$ is the uniformly random unit vector $(z_1,\dots,z_{n_2})$. The adversary can win only if the distance $\Delta$ of $x'$ from $S$ is at most $\frac{\tau}{r'}$. Therefore a threshold change of $\tau$ to $\tau+d\tau$ corresponds to $\Delta\in\left(\frac{\tau}{r'},\frac{\tau+d\tau}{r'}\right)$. We observe from the proof of Lemma \ref{lem:hsphere} that
\begin{align*}
    \Pr\left[\Delta\in\left(\frac{\tau}{r'},\frac{\tau+d\tau}{r'}\right)\right]&=C(n_2,n_3)\cdot \int_{\tau/r'}^{(\tau+d\tau)/r'}\rho^{n_2-n_3-1}\left(\sqrt{1-\rho^2}\right)^{n_3-2}d\rho\\&\le C(n_2,n_3)\cdot \frac{\tau^{n_2-n_3-1}d\tau}{r'^{n_2-n_3}},
\end{align*}
where $C(n_2,n_3)=2A(n_3-1)A(n_2-n_3-1)$ is a constant for fixed dimensions $n_2,n_3$. This holds for any test point $\x\in\cT$, and in particular, in average over the test points. Using a union bound over training points we conclude,
\[\cE_\adv(\tau+d\tau)-\cE_\adv(\tau)\le mC(n_2,n_3)\frac{\tau^{n_2-n_3-1}d\tau}{r'^{n_2-n_3}}.\]
The slope bound increases with $\tau$, substituting $\tau\le{\tau_{\max}}$ and $r'\ge r$ gives the desired bound on Lipschitzness.
\end{proof}

\begin{corollary}\label{cor:bounded density}
For a $\Tilde{\kappa}$-bounded adversary distribution $\cS$ in Lemma \ref{lem:acc-lipschitzness}, we have that $\cE_\adv(\tau)$ is $O\left(\Tilde{\kappa}m\tau_{\max}^{n_2-n_3-1}/r^{n_2-n_3}\right)$-Lipshcitz.
\end{corollary} 
\begin{proof}
The proof follows using the same arguments in the proof of Theorem \ref{theorem: bounded density} applied to Lemma \ref{lem:acc-lipschitzness} (instead of our upper bounds on the robust error).
\end{proof}
\blue{
\begin{corollary}\label{cor:estimated-loss}
For $S$ drawn from a $\Tilde{\kappa}$-bounded adversary distribution $\cS$, the expected number of discontinuities of $\cE_\adv(\tau,S)$ in any $\tau$-interval of length $w$ is at most $O\left(\Tilde{\kappa}bmw\tau_{\max}^{n_2-n_3-1}/r^{n_2-n_3}\right)$.
\end{corollary}

\begin{proof}
Consider the interval $[\tau,\tau+w]$. We are interested in bounding the probability that for a given test point $\x$, the smallest threshold $\tau'$ for which the adversary succeeds when perturbing along $S$ (over the draw $S\sim\cS$) lies in the interval $[\tau,\tau+w]$.

For a fixed training point $\x_i$, the probability of adversarial success on any $\x\in\cT$ by perturbing to a point at distance $\tau'\in[\tau,\tau+w]$ from $\x_i$ is  bounded by $O\left(\Tilde{\kappa}w\tau_{\max}^{n_2-n_3-1}/r^{n_2-n_3}\right)$ as argued above (Lemma \ref{lem:acc-lipschitzness}). Taking a union bound over training points $\x_i$ implies the adversary succeeds with probability at most $O\left(\Tilde{\kappa}mw\tau_{\max}^{n_2-n_3-1}/r^{n_2-n_3}\right)$ by perturbing to within $[\tau,\tau+w]$ of some training point. Thus, for $b$ test points the expected number of breakpoints is at most $O\left(\Tilde{\kappa}bmw\tau_{\max}^{n_2-n_3-1}/r^{n_2-n_3}\right)$.

\end{proof}
}

\noindent\blue{The following lemma gives a bound on the expected number of breakpoints in $\cD_\nat(\tau)$, a piecewise constant function in $\tau$, in a small interval of width $w$.}

\begin{lemma}\label{lem:dk-dispersion}
Suppose that the data distribution satisfies the assumptions in Lemma \ref{lem:small-tau}, and further is $\kappa$-bounded. The expected number of discontinuties in $\cD_\nat(\tau)$ in any interval of width $w$ for $\tau\le \tau_{\max}$ is $O(\kappa bmw \tau_{\max}^{n_2-1})$.
\end{lemma}
\begin{proof} Note that the discontinuities of $\cD_\nat(\tau)$ in an interval $(\tau,\tau+w)$ corresponds to points $(\x,\y)\in T$ such that nearest neighbor distance of $\x$ is in that interval.
\begin{align*}
    E[\text{number of discontinuities in }(\tau,\tau+w)]&=b \Pr[\text{nearest neighbor of a test point }\in (\tau,\tau+w)]\\
    &\le b \Pr[\text{some neighbor of a test point }\in (\tau,\tau+w)]\\
    &\le\kappa bm \vol(\text{spherical shell of radius }\tau\text{ and width }w)\\
    &=\kappa bmO(\tau_{\max}^{n_2-1}w)\\
    &=O(\kappa bmw \tau_{\max}^{n_2-1}).
\end{align*}
\end{proof}

\noindent For the full proof of Theorem \ref{thm:regret}, we will need  a low-regret bound for dispersed functions due to \cite{balcan2018dispersion}.  If the sequence of functions is dispersed (Definition \ref{def:dispersion}), we can bound the regret of a simple exponential forecaster algorithm (Algorithm \ref{alg:ef}) by the following theorem.

\begin{theorem}[\cite{balcan2018dispersion}]\label{thm:dispersion}
Let $u_1 ,\dots , u_T : C \rightarrow [0, 1]$ be any sequence of piecewise $L$-Lipschitz functions that are $(w,k)$-dispersed. Suppose $C \subset \R^d$ is contained in a ball of radius $R$ and $B(\rho^*,w)\subset C$, where $\rho^* = \argmax_{\rho\in C} \sum_{i=1}^T u_i(\rho)$. The exponentially weighted forecaster with $\lambda = \sqrt{d \ln(R/w)/T}$ has expected regret bounded by
$O\left(\sqrt{Td\log(R/w)}+k +TLw\right)$.
\end{theorem}



\section{Estimating Point-Specific Threshold of ``Don't Know''}
Algorithm \ref{algorithm: robust inference-time classifier with data-specific threshold} gives an alternative to our algorithm where instead of using a fixed threshold for each point, we use a variable point-specific threshold learned from the data. For this algorithm, we have the following guarantee.

\begin{algorithm}[t]
\caption{Robust classifier in the feature space with point-specific threshold $\tau_i^\cA$ of ``don't know''}
\label{algorithm: robust inference-time classifier with data-specific threshold}
\begin{algorithmic}[1]
\STATE {\bfseries Input:} A test example $F(\x)$ (potentially an adversarial example), a set $\cA$ of training examples $F(\x_i^\cA)$ and their labels $y_i^\cA$, $i\in[m_\cA]$, a set $\cB$ of training examples $F(\x_i^\cB)$ and their labels $y_i^\cB$, $i\in[m_\cB]$.
\STATE {\bfseries Output:} A predicted label of $F(\x)$, or ``don't know''.
\STATE{$\tau_i^\cA\leftarrow \min_{j:\ y_i^\cA\not=y_j^\cB} \dist(F(\x_i^\cA),F(\x_j^\cB))$ for all $i\in [m_\cA]$.}
\STATE{$i_{\min}\leftarrow \argmin_{i\in[m]} \dist(F(\x),F(\x_i^\cA))$.}
\IF{$\dist(F(\x),F(\x_{i_{\min}}^\cA))<\tau_{i_{\min}}^\cA$}
\STATE{\textbf{return} $y_{i_{\min}}^\cA$;}
\ELSE
\STATE{\textbf{return} ``don't know''.}
\ENDIF
\end{algorithmic}
\end{algorithm}

\begin{theorem}
Suppose that the sets $\cA$ and $\cB$ are two independent samples from $F(\cX)$ of size $m_\cA$ and $m_\cB$, respectively. Let $m_\cB=\frac{m_\cA}{\epsilon\delta}$. Then with probability at least $1-\delta$ over the draw of $\cA$, for a new sample $F(\x)$, the probability that ``\blue{there exists $F(\x^\cA) \in \cA$ such that $F(\x)$ is closer to $F(\x^\cA)$ than any point in $\cB$ of different labels than $F(\x^\cA)$, and $F(\x)$ has a different label than $F(\x^\cA)$}'' is at most $\epsilon$, where the probability is taken over the draw of $F(\x)$ and the draw of $\cB$.
\end{theorem}

\begin{proof}
 Fixing the draw of set $\cA$, we can think of picking a random set $\cS$ of size $m_\cB+1$ and randomly choosing one of the points in it to be $F(\x)$ and the rest to be $\cB$. \blue{Let $F(\x^\cA)$ be an arbitrary point in $\cA$.} Assuming $\cS$ has at least one point in it of a different label than $F(\x^\cA)$, then there is exactly a $\frac{1}{m_\cB+1}$ probability that we choose $F(\x)$ to be the closest point in $\cS$ to $F(\x^\cA)$ of a different label than $F(\x^\cA)$; if $\cS$ has all points of the same label as $\x^\cA$, then the probability is 0. Now we can apply the union bound over all $F(\x^\cA)$ in $\cA$ to get a total probability of failure at most $\frac{m_\cA}{m_\cB+1} < \epsilon \delta$.

The above analysis gives an expected failure probability over the draw of set $\cA$. Applying the Markov inequality gives a high-probability bound.
\end{proof}

\section{Technical Lemmas for Proof of Theorem \ref{theorem: toy example optimal threshold}}
\label{appendix: toy example}

\begin{lemma}\label{lemma: toy example}
In the setting of Theorem \ref{theorem: toy example optimal threshold}, $l(\tau)$ is monotonically non-decreasing for $\tau>D$.
\end{lemma}
\begin{proof}
Note that $\cD_\nat(\tau)=0$ for $\tau>D$ as long as $m>0$, since any test point of a class must be within $D$ of every training point of that class. Hence, it suffices to note that $\cE_\adv(\tau)$ is monotonically non-decreasing in $\tau$ (increasing the threshold can only increase the ability of the adversary to successfully perturb to the opposite class).
\end{proof}

\begin{lemma}\label{lemma: toy example abstain}
In the setting of Theorem \ref{theorem: toy example optimal threshold}, the abstention rate is given by
\begin{align*}
    \cD_\nat(\tau)=\frac{1}{m+1}\left[2\ind_{\tau\le D}\left(1-\frac{\tau}{D}\right)^{m+1}+(m-1)\ind_{\tau\le D/2}\left(1-\frac{2\tau}{D}\right)^{m+1}\right].
\end{align*}
\end{lemma}
\begin{proof}
Note that for $\tau\ge D$, if $m>0$, we never abstain on any test point. So we will assume $\tau\le D$ in the following. Consider a test point $\x=(x,0)$ sampled from class $A$ (class $B$ is symmetric, so the overall abstention rate is the same is that of points drawn from class $A$). Let $\mathrm{nbd}_{\x}(\tau)$ denote the intersection of a ball of radius $\tau$ around $x$ with $S_A$. For $x$ to be classified as `don't know', we must have no training point sampled from $\mathrm{nbd}_{\x}(\tau)$. This happens with probability $\left(1-\frac{|\mathrm{nbd}_{\x}(\tau)|}{D}\right)^m$, where $|\mathrm{nbd}_{\x}(\tau)|$ is the size of $\mathrm{nbd}_{\x}(\tau)$ and is given by
\begin{align*}
    |\mathrm{nbd}_{\x}(\tau)|=\begin{cases}
    \min\{x+\tau,D\},&x<\tau;\\
    \min\{2\tau,D\},&\tau\le x\le D-\tau;\\
    \min\{D-x+\tau,D\},&x>D-\tau.
    \end{cases}
\end{align*}

\noindent Averaging over the distribution of test points $\x$, we get
\begin{align*}
    \cD_\nat(\tau)&=\frac{1}{D}\int_0^{D}\left(1-\frac{|\mathrm{nbd}_{\x}(\tau)|}{D}\right)^mdx\\
    &=\frac{1}{m+1}\left[2\left(1-\frac{\tau}{D}\right)^{m+1}+(m-1)\ind_{\tau\le D/2}\left(1-\frac{2\tau}{D}\right)^{m+1}\right].
\end{align*}
\end{proof}

\begin{lemma}\label{lemma: toy example accuracy}
In the setting of Theorem \ref{theorem: toy example optimal threshold}, the robust accuracy rate for $\tau\le D$ is given by
\begin{align*}
    \cA_\adv(\tau)=1-\frac{\tau}{\pi r}\left(1-\frac{m+3}{m+1}\cdot\Theta\left(\frac{D}{r}\right)\right) -\Theta\left(\left(\frac{\tau}{r}\right)^3\right).
\end{align*}
\end{lemma}
\begin{proof}
Consider a test point $\x=(x,0)$ from $S_A$. Let $\y=(y,0)$ denote the nearest point in $S_B$. In the given geometry, it is easy to see that if $\x$ can be perturbed into the $\tau$ neighborhood of some point $\y'\in S_B$ when moved along a fixed direction, then it must be possible to perturb it into the $\tau$ neighborhood of $\y$ (Figure \ref{figure:toy-example-appendix}). Therefore it suffices to consider directions where perturbation to the $\tau$-ball around $\y$ is possible.

\begin{figure}[t]
\center

\tikzset{every picture/.style={line width=0.75pt}} 

\begin{tikzpicture}[x=0.75pt,y=0.75pt,yscale=-1,xscale=1]

\draw [color={rgb, 255:red, 45; green, 19; blue, 241 }  ,draw opacity=1 ][line width=2.25]    (485.5,69.69) -- (421.5,69.69) ;
\draw    (182.5,69.19) -- (419.5,69.19) ;
\draw [shift={(421.5,69.19)}, rotate = 180] [color={rgb, 255:red, 0; green, 0; blue, 0 }  ][line width=0.75]    (10.93,-3.29) .. controls (6.95,-1.4) and (3.31,-0.3) .. (0,0) .. controls (3.31,0.3) and (6.95,1.4) .. (10.93,3.29)   ;
\draw [shift={(180.5,69.19)}, rotate = 0] [color={rgb, 255:red, 0; green, 0; blue, 0 }  ][line width=0.75]    (10.93,-3.29) .. controls (6.95,-1.4) and (3.31,-0.3) .. (0,0) .. controls (3.31,0.3) and (6.95,1.4) .. (10.93,3.29)   ;
\draw  [draw opacity=0][fill={rgb, 255:red, 25; green, 163; blue, 226 }  ,fill opacity=1 ] (435.25,64.69) -- (440.5,73) -- (430,73) -- cycle ;
\draw  [draw opacity=0][fill={rgb, 255:red, 240; green, 101; blue, 37 }  ,fill opacity=1 ] (148.5,66.19) -- (155,66.19) -- (155,72.69) -- (148.5,72.69) -- cycle ;
\draw  [draw opacity=0][fill={rgb, 255:red, 25; green, 163; blue, 226 }  ,fill opacity=1 ] (467.25,64.69) -- (472.5,73) -- (462,73) -- cycle ;
\draw [color={rgb, 255:red, 240; green, 101; blue, 37 }  ,draw opacity=1 ][line width=2.25]    (180.5,69.19) -- (116.5,69.19) ;
\draw    (51.75,75.32) -- (540.5,46.8) ;
\draw [shift={(542.5,46.69)}, rotate = 536.6600000000001] [color={rgb, 255:red, 0; green, 0; blue, 0 }  ][line width=0.75]    (10.93,-3.29) .. controls (6.95,-1.4) and (3.31,-0.3) .. (0,0) .. controls (3.31,0.3) and (6.95,1.4) .. (10.93,3.29)   ;
\draw [shift={(49.75,75.44)}, rotate = 356.66] [color={rgb, 255:red, 0; green, 0; blue, 0 }  ][line width=0.75]    (10.93,-3.29) .. controls (6.95,-1.4) and (3.31,-0.3) .. (0,0) .. controls (3.31,0.3) and (6.95,1.4) .. (10.93,3.29)   ;
\draw   (410.25,68.84) .. controls (410.25,55.04) and (421.44,43.84) .. (435.25,43.84) .. controls (449.06,43.84) and (460.25,55.04) .. (460.25,68.84) .. controls (460.25,82.65) and (449.06,93.84) .. (435.25,93.84) .. controls (421.44,93.84) and (410.25,82.65) .. (410.25,68.84) -- cycle ;
\draw   (442.25,68.84) .. controls (442.25,55.04) and (453.44,43.84) .. (467.25,43.84) .. controls (481.06,43.84) and (492.25,55.04) .. (492.25,68.84) .. controls (492.25,82.65) and (481.06,93.84) .. (467.25,93.84) .. controls (453.44,93.84) and (442.25,82.65) .. (442.25,68.84) -- cycle ;

\draw (290,75) node [anchor=north west][inner sep=0.75pt]   [align=left] {$\displaystyle r$};
\draw (126,102) node [anchor=north west][inner sep=0.75pt]   [align=left] {Class A};
\draw (424,104) node [anchor=north west][inner sep=0.75pt]   [align=left] {Class B};
\draw (147,50) node [anchor=north west][inner sep=0.75pt]   [align=left] {$\displaystyle \mathbf{x}$};
\draw (423,77) node [anchor=north west][inner sep=0.75pt]   [align=left] {$\displaystyle \mathbf{y}$};
\draw (469,73) node [anchor=north west][inner sep=0.75pt]   [align=left] {$\displaystyle \mathbf{y} '$};
\draw (260,35) node [anchor=north west][inner sep=0.75pt]   [align=left] {adversarial direction};
\draw (430,20) node [anchor=north west][inner sep=0.75pt]   [align=left] {$\tau$ balls};

\end{tikzpicture}
\caption{It suffices to consider the nearest point of the opposite class for adversarial perturbation.\label{figure:toy-example-appendix}}
\end{figure}

Therefore the probability of adversary's success for $\x$, given $\y$ is the nearest point of the opposite class, is
\[\mathrm{err}_{\x\mid\y}(\tau)=\frac{1}{\pi}\arcsin\left(\frac{\tau}{y-x}\right)=\frac{1}{\pi}\arcsin\left(\frac{\tau}{r+d}\right),\]
where $d=y-x-r\in [0,2D]$. Now since $\tau\le D=o(r)$, we have
\[\mathrm{err}_{\x\mid\y}(\tau)=\frac{\tau}{\pi(r+d)}+\Theta\left(\left(\frac{\tau}{r}\right)^3\right)=\frac{\tau}{\pi r}\left(1-\Theta\left(\frac{d}{r}\right)\right) +\Theta\left(\left(\frac{\tau}{r}\right)^3\right).\]

\noindent We can now compute the average error using the probability distributions for $\x$ and $\y$, $\x$ is a uniform distribution over $S_A$, while $\y$ is a nearest-neighbor distribution.
\[p(x)=\frac{1}{D},\;\; p(y)=\frac{m}{D}\left(1-\frac{y-r-D}{D}\right)^{m-1}.\]

\noindent The average value of $d$ is
\[\Bar{d}=\int_0^D\int_0^D(y'+x')\frac{m}{D}\left(1-\frac{y'}{D}\right)^{m-1}dy'\frac{dx'}{D}=\frac{D(m+3)}{2(m+1)}.\]
Using this to compute the average of $\mathrm{err}_{\x\mid\y}(\tau)$ gives the result.
\end{proof}



\section{Additional Experiments}

\begin{algorithm}[t]
\caption{Approximate computation of attacks under threat model \ref{sec: threat model} against Algorithm \ref{algorithm: robust separated inference-time classifier}}
\label{algorithm: adversarial attack}
\begin{algorithmic}[1]
\STATE {\bfseries Input:} A randomly-sampled adversarial subspace $S$ of dimension $n_3$, a test example $F(\x)$ and its label $y$, a set of training examples $F(\x_i)$ and their labels $y_i$, $i\in[m]$, a threshold parameter $\tau$.
\STATE {\bfseries Output:} A misclassified adversarial example $F(\x)+\v$, $\v\in S$ if it exists; otherwise, output ``no adversarial example found''.
\STATE{$F_{\mathrm{center}}(\x_i)\leftarrow F(\x_i)-F(\x)$ for $i\in[m]$.}\\
{(The $F_{\mathrm{proj}}(\x_i)$'s in the next step are candidate adversarial examples.)}
\STATE{Project $F_{\mathrm{center}}(\x_i)$, $i\in[m]$ onto $S$ and obtain $F_{\mathrm{proj}}(\x_i)$ for $i\in[m]$.}
\FOR{$i=1,...,m$}
\STATE{Run the nearest-neighbor algorithm to predict the label of $F_{\mathrm{proj}}(\x_i)$ with the training set $\{(F_{\mathrm{center}}(\x_j),y_j):j=1,...,m\}$.}
\IF{the output of the nearest-neighbor algorithm is NOT $y$ \textbf{and} the closest distance is smaller than $\tau$}
\STATE{\textbf{return} $F(\x)+F_{\mathrm{proj}}(\x_i)$.}
\ENDIF
\ENDFOR
\STATE{\textbf{return} ``no adversarial example found''.}
\end{algorithmic}
\end{algorithm}

\subsection{Approximating Robust Accuracy for Large $n_3$}\label{app:attack-large-n3}
The experiments in Section \ref{sec:expt} consider an adversary which is difficult to compute in practice for large adversarial space, that is large $n_3$. In this section we present a `greedy' adversary (Algorithm \ref{algorithm: adversarial attack}) which provides a good approximation to the exact adversary for small $\tau$, which can be easily run even for large $n_3$: we can generate the adversarial examples of $F(\x)$ by projecting each training example onto the affine subspace $F(\x)+S$ and pick the one with the closest distance to $F(\x)$. We denote the accuracy against this algorithm as $\hat{\cA}_\adv$. The averaged results of multiple runs are in Table \ref{table: robustness}: we report the \emph{natural accuracy} ($\cA_\nat=1-\cE_\nat$), the \emph{adversarial accuracy}, and the \emph{abstention rate}, where the \emph{abstention rate} represents the fraction of algorithm's output of ``don't know'' among the misclassified data by the nearest-neighbor classifier.

\begin{table*}[t]
	\centering
	\caption{Natural accuracy $\cA_\nat$ and adversarial accuracy $\hat{\cA}_\adv$ on the CIFAR-10 data set when the 512-dimensional representations are learned by contrastive learning, where \emph{abstain} represents the fraction of each algorithm's output of ``don't know'' among the misclassified data by ours ($\tau\rightarrow\infty$, a.k.a. the nearest-neighbor classifier).}
	\label{table: robustness}
	\resizebox{1.0\textwidth}{!}{%
\begin{tabular}{@{}cr||cc|cc|cc|cc@{}}
\toprule
& \multirow{2}{2cm}{Contrastive} & \multicolumn{2}{c}{Linear Protocol} & \multicolumn{2}{c}{Ours ($\tau\rightarrow\infty$)} & \multicolumn{2}{c}{Ours ($\tau=1.0$)} & \multicolumn{2}{c}{Ours ($\tau=0.8$)}\\
& & $\cA_{\nat}$ & $\hat{\cA}_\adv$ & $\cA_{\nat}$ & $\hat{\cA}_\adv$ & $\cA_{\nat}$/abstain & $\hat{\cA}_\adv$/abstain & $\cA_{\nat}$/abstain & $\hat{\cA}_\adv$/abstain\\
\midrule
\midrule
\multirow{2}{1.5cm}{$n_3=1$} & Self-supervised & 91.1\% & 0.0\% & 84.5\% & 81.5\% & 99.3\%/95.5\% & 99.2\%/95.7\% & 100.0\%/100.0\% & 100.0\%/100.0\%\\
& Supervised & 94.4\% & 0.0\% & 94.3\% & 93.5\% & 95.0\%/12.3\% & 94.5\%/15.4\% & 97.7\%/59.6\% & 97.7\%/64.6\%\\
\midrule
\multirow{2}{1.5cm}{$n_3=25$} & Self-supervised & 91.1\% & 0.0\% & 84.5\% & 65.1\% & 99.3\%/95.5\% & 98.8\%/96.6\% & 100.0\%/100.0\% & 100.0\%/100.0\%\\
& Supervised & 94.4\% & 0.0\% & 94.3\% & 84.5\% & 95.0\%/12.3\% & 91.6\%/45.8\% & 97.7\%/59.6\% & 96.8\%/79.4\%\\
\midrule
\multirow{2}{1.5cm}{$n_3=50$} & Self-supervised & 91.1\% & 0.0\% & 84.5\% & 56.3\% & 99.3\%/95.5\% & 98.3\%/96.1\% & 100.0\%/100.0\% & 100.0\%/100.0\%\\
& Supervised & 94.4\% & 0.0\% & 94.3\% & 71.7\% & 95.0\%/12.3\% & 89.7\%/63.6\% & 97.7\%/59.6\% & 95.5\%/84.1\%\\
\midrule
\multirow{2}{1.5cm}{$n_3=100$} & Self-supervised & 91.1\% & 0.0\% & 84.5\% & 31.1\% & 99.3\%/95.5\% & 96.7\%/95.2\% & 100.0\%/100.0\% & 99.7\%/99.6\%\\
& Supervised & 94.4\% & 0.0\% & 94.3\% & 35.0\% & 95.0\%/12.3\% & 86.3\%/78.9\%  & 97.7\%/59.6\% & 93.0\%/89.2\%\\
\midrule
\multirow{2}{1.5cm}{$n_3=200$} & Self-supervised & 91.1\% & 0.0\% & 84.5\% & 1.2\% & 99.3\%/95.5\% & 91.1\%/91.0\% & 100.0\%/100.0\% & 98.6\%/98.6\%\\
& Supervised & 94.4\% & 0.0\% & 94.3\% & 0.7\% & 95.0\%/12.3\% & 74.7\%/74.5\% & 97.7\%/59.6\% & 85.8\%/85.7\%\\
\bottomrule
\end{tabular}
}
\end{table*}

\begin{figure}[b]
\centering
\subfigure{
\includegraphics[scale=0.3]{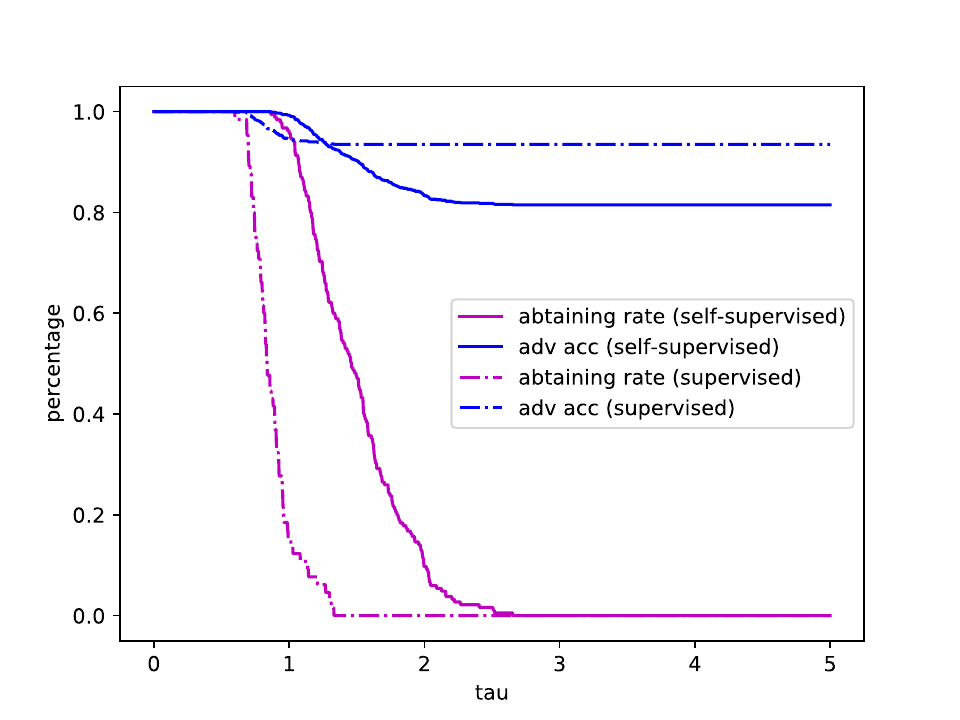}}
\subfigure{
\includegraphics[scale=0.3]{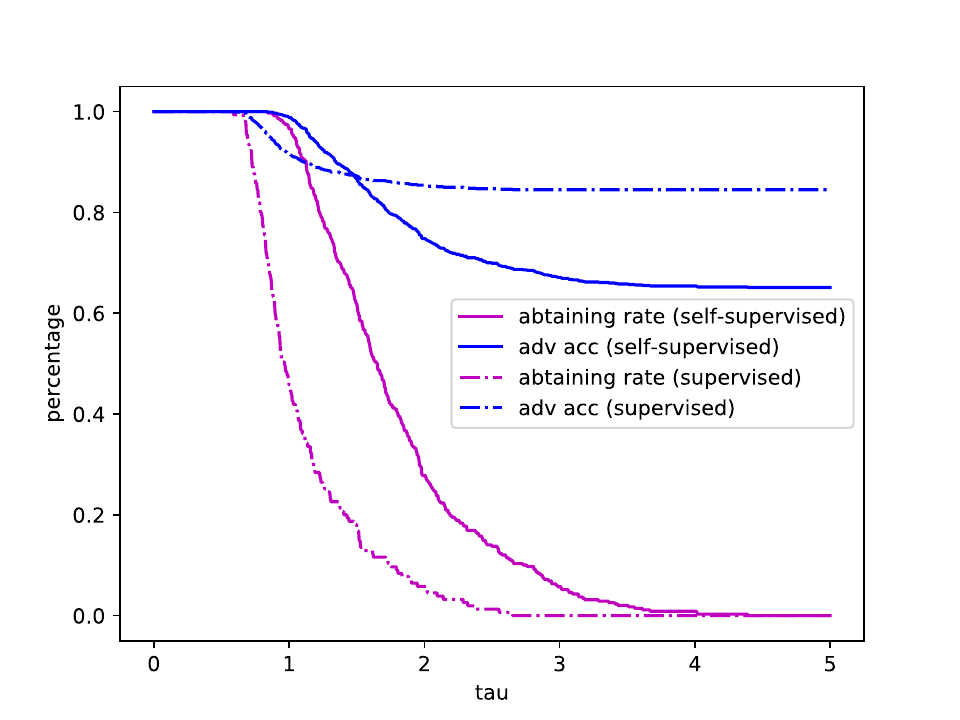}}
\subfigure{
\includegraphics[scale=0.3]{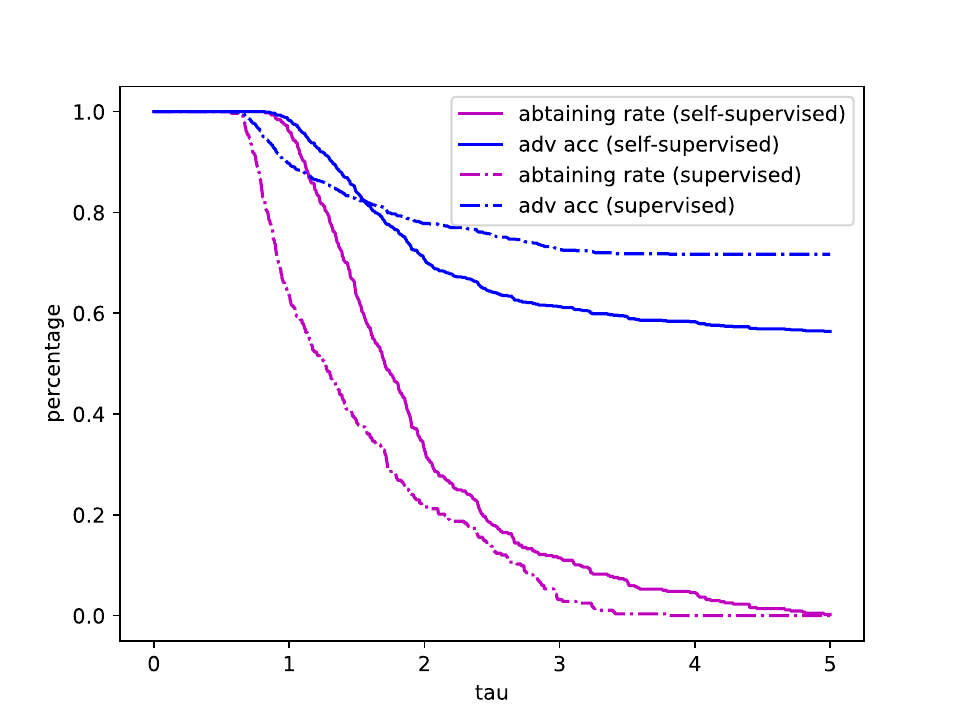}}
\caption{Sensitivity of model success rate (estimated by $\hat{\cA}_\adv$) and abstention rate on the parameter $\tau$, where \emph{abstain} represents the fraction of algorithm's output of ``don't know'' among the misclassified data by ours ($\tau\rightarrow\infty$, a.k.a. the nearest-neighbor classifier). \textbf{Left Figure:} $n_3=1$. \textbf{Middle Figure:} $n_3=25$. \textbf{Right Figure:} $n_3=50$.}
\label{figure: sensitivity}
\end{figure}
We observe that as the dimension of adversarial subspaces $n_3$ increases, the adversarial accuracy $\hat{\cA}_\adv$ decreases while the abstention rate tends to increase, which verifies an intrinsic trade-off between robustness and abstention rate. Recall that our algorithm abstains if and only if the closest distance in feature space between the given test example and any training example is larger than a threshold $\tau$. As the threshold parameter $\tau$ decreases, the adversarial accuracy $\hat{\cA}_\adv$ increases while the algorithm abstains from predicting the class of more data.

\subsubsection{Sensitivity of Threshold Parameter $\tau$}

The threshold parameter $\tau$ is an important hyperparameter in our proposed method. It captures the trade-off between the accuracy and the abstention rate. We show how the threshold parameter affects the performance of our robust classifiers by numerical experiments on the CIFAR-10 data set. We first train a embedding function $F$ by following the setups in Section \ref{section: certified adversarial robustness}. We then fix $F$ and run our evaluation protocol by varying $\tau$ from 0.0 to 5.0 with step size 0.001. We summarize our results in Figure \ref{figure: sensitivity} which plots the adversarial accuracy $\hat{\cA}_\adv$ and the abstention rate for three representative dimension of adversarial subspace. Compared with self-supervised contrastive learning (the solid line), supervised contrastive learning (the dashed line) enjoys higher adversarial accuracy (the blue curve) and smaller abstention rate (the red curve) for fixed $\tau$'s due to the use of extra label information. For both setups, the adversarial accuracy is not very sensitive to the choice of $\tau$.

\end{document}